%% file: main.tex
\documentclass{article}

\usepackage{iclr2027_conference,times}
\input{math_commands.tex}

\usepackage{amsmath,amssymb,amsthm,mathtools,bm}
\usepackage{booktabs}
\usepackage{graphicx}
\usepackage{float}
\usepackage{placeins}
\usepackage{microtype}
\usepackage{hyperref}
\usepackage{url}

\newtheorem{theorem}{Theorem}[section]
\newtheorem{proposition}[theorem]{Proposition}
\newtheorem{lemma}[theorem]{Lemma}
\newtheorem{corollary}[theorem]{Corollary}
\theoremstyle{definition}

\theoremstyle{remark}

\title{The Sequential Price of Continual Learning}

\author{Zonghuan Xu\\
Fudan University\And
Xingjun Ma\\
Fudan University}

\iclrfinalcopy
\hypersetup{pdftitle={The Sequential Price of Continual Learning},pdfauthor={Zonghuan Xu, Xingjun Ma}}
\begin{document}

\maketitle
\lhead{Preprint}

\begin{abstract}
Sequential task updates are fundamental to continual learning, but their
recency bias can impose a lasting performance cost. We study this cost in an
overparameterized linear-regression model with i.i.d. task sampling. We prove
that distribution-level forgetting and population loss converge to the same
stationary limit. We quantify the additional loss incurred by sequential exact
fitting, or the \emph{sequential price}. If all nonzero eigenvalues of all task curvature matrices equal the same constant, the stationary loss is exactly twice the joint-training optimum. We further analyze
fixed-strength elastic weight consolidation (EWC) under general task curvatures and characterize its stationary
sequential price at every regularization strength. Under strong regularization,
the price decays inversely with EWC strength while the number of tasks
sufficient to reduce the mean-square effect of initialization by a fixed
factor grows proportionally. Real Jester user ratings, with a natural
encoding and training procedure, satisfy our theoretical assumptions.
\end{abstract}

\section{Introduction}
\label{sec:introduction}

Continual learning studies how models learn from tasks that arrive over time
\citep{delange2021continual,wang2024survey,mai2022survey}.
A model updates its current parameters on each new task and then uses the
updated parameters as the starting point for learning subsequent tasks. Its
capabilities develop through this continuing learning process. Sequential
learning is therefore a basic process in continual learning, through which
tasks successively act on a model and together shape its behavior over time.
As models and agents need to keep learning over longer periods, understanding
the long-term behavior of this process becomes increasingly important.

In sequential learning, tasks encountered at different times influence the
current model in different ways. Each update responds directly to the current
task, while the influence of earlier tasks passes through a series of
subsequent updates. This can bias the model toward recent tasks and degrade
its performance on past tasks, giving rise to forgetting
\citep{parisi2019continual,mirzadeh2020understanding,lesort2023challenging}.
This recency bias also suggests that sequential learning may have a lasting
effect on the model's overall performance. We therefore ask how sequential
learning affects the model's long-term population loss over the task
distribution.

One way to study this question is to view arriving tasks as independent
samples from the same task distribution. The distribution provides a basis
for describing the properties of task sequences as a whole and allows us to
study average forgetting at the distribution level, starting from forgetting
along individual training histories. Prior work, \emph{From Order to
Distribution} \citep{xu2026order}, adopts this perspective and connects the
learning dynamics along concrete task sequences to properties of the task
distribution. This provides a starting point for our analysis of the
long-term behavior of sequential learning.

That work considers tasks that share a common solution and proves that
distribution-level forgetting converges to zero. In this setting, tasks may
have different training data, but a single parameter vector can fit all tasks
exactly \citep{evron2022catastrophic,xu2026order}. More generally, tasks may
each admit an exact fit while requiring conflicting predictions, so that no
common solution exists. This raises a follow-up question: how does
distribution-level forgetting evolve as training continues when tasks are
incompatible in this way?

We first establish a connection between forgetting and population loss.
Distribution-level forgetting evaluates the current model through its average
loss on past tasks, while population loss evaluates its expected loss over the
task distribution. For conflicting tasks, both quantities may remain positive
in the long run. We prove that, under i.i.d. task sampling and the stability
conditions considered here, they converge to the same stationary limit, with
a gap that decays at rate $O(T^{-1})$. The dependence between an earlier task
and the current model weakens as the number of intervening tasks increases,
while recent tasks that remain strongly dependent account for a diminishing
fraction of the full history. Thus, the long-term limit of forgetting also
characterizes the stationary population loss of sequential learning.

This connection allows us to compare the long-term limit of forgetting
directly with $R^\star$, the minimum population loss within the same model
class. For any learning procedure that outputs a single shared parameter
vector, $R^\star$ is a lower bound on its population loss
(Lemma~\ref{lem:population-optimum}), and joint training asymptotically
attains this level as the number of tasks grows
(Lemma~\ref{lem:joint-consistency}). We call the excess of long-term population loss under sequential learning
over this intrinsic loss the \emph{sequential price}, and characterize its
sources and magnitude under sequential exact fitting. This decomposes the long-term loss into the
unavoidable loss caused by task incompatibility and the additional cost of
fitting these tasks sequentially. When all nonzero eigenvalues of all task
curvature matrices equal the same constant, the stationary population loss is
exactly twice the joint-training optimum.

We next use the sequential price to analyze fixed-strength EWC under general
task curvatures
\citep{kirkpatrick2017overcoming,schwarz2018progress,benzing2022unifying},
quantifying how regularization changes long-term loss and the decay of
initialization dependence. For every fixed regularization strength, the
stationary sequential price can be computed from the task distribution's
curvatures and targets. In the strong-regularization regime, this price
decreases as $1/\lambda$, while the number of tasks sufficient to reduce the
mean-square difference due to different initializations to a fixed fraction
of its initial value grows as $\lambda$.

Finally, we quantitatively validate the theoretical results on Jester, a
real-world user-rating dataset. We treat each user and their observed ratings
as one task and use a linear model and task-sampling procedure that satisfy
our theoretical assumptions. In this setting, theoretical predictions can be
computed exactly from the full empirical task distribution and compared with
results from independently sampled task streams. The two agree on the
convergence of forgetting and population loss, the decomposition of the
sequential price, and the effects of EWC. Despite variation in curvature
across user tasks, the long-term population loss under sequential exact
fitting is approximately twice the joint-training optimum.

Our contributions are as follows.
\begin{itemize}
    \item We show that forgetting and population loss converge to the same limit under sequential exact fitting and use joint training to isolate the sequential price, so that continual-learning methods should be compared by how much of this additional loss they remove.
    \item We characterize how displacement of the stationary mean and parameter fluctuations contribute to the sequential price, bound it in terms of the intrinsic loss, and show that a natural encoding and training procedure for real Jester ratings satisfy the assumptions of this analysis.
    \item We prove that as the EWC regularization strength $\lambda$ grows, the remaining sequential price is $a_\Pi/\lambda+O(\lambda^{-2})$, where $a_\Pi\ge 0$ depends on the task distribution and is strictly positive when tasks conflict.
\end{itemize}

\section{Setup}
\label{sec:setup}

\subsection{Task model and learning dynamics}
\label{sec:projection-dynamics}
\label{sec:regularity}

To analyze the long-term behavior of forgetting and population loss when
tasks conflict, we use a classical linear-regression model from continual
learning theory
\citep{evron2022catastrophic,ding2024understanding,goldfarb2024joint}.
Each task can be fitted exactly on its own, but the tasks may have no common
solution. In this model, gradient descent run to convergence implements an
orthogonal projection, giving a parameter update that we can analyze.

Task $\tau$ has data $D_\tau=(X_\tau,y_\tau)$, where
$X_\tau\in\mathbb R^{n_\tau\times d}$ and $y_\tau\in\mathbb R^{n_\tau}$.
All tasks share a parameter vector $w\in\mathbb R^d$, with task loss
\[
\ell_\tau(w)=\frac{1}{n_\tau}\|X_\tau w-y_\tau\|_2^2.
\]
We assume that $X_\tau w=y_\tau$ has a solution. In the overparameterized
setting, this solution is typically not unique. Starting from the current
parameter $u$, gradient descent with a suitable step size, run to convergence,
selects the exact fit closest to $u$ in Euclidean distance, denoted by
$S_\tau(u)$ \citep{evron2022catastrophic,lin2023theory,karpel2026l2}.

This update has a direct geometric interpretation. All exact fits have the
same component in the row space of $X_\tau$, denoted by $Z_\tau$, while their
components in the kernel do not affect predictions on the current task.
The closest exact fit therefore replaces the row-space component of $u$ with
$Z_\tau$ and retains its kernel component. Letting $P_\tau$ be the orthogonal
projector onto $\ker X_\tau$, we obtain
\[
S_\tau(u)=P_\tau u+Z_\tau.
\]
Here $P_\tau=I-X_\tau^\dagger X_\tau$ and
$Z_\tau=X_\tau^\dagger y_\tau$, where $\dagger$ denotes the Moore--Penrose
pseudoinverse.

Tasks $\tau_1,\tau_2,\ldots$ are drawn independently from a fixed distribution
$\Pi$ \citep{xu2026order}, and the initialization $W_0$ is independent of the
task sequence. Fitting these tasks exactly in sequence gives the recursion
\[
W_t=P_{\tau_t}W_{t-1}+Z_{\tau_t}.
\]
We abbreviate $P_t=P_{\tau_t}$ and $Z_t=Z_{\tau_t}$.

\paragraph{\normalfont Regularity assumptions.}
Write $C_\tau=X_\tau^\top X_\tau/n_\tau$ for task curvature and
$\overline C=\mathbb E[C_\tau]$ for its mean, with $C_t=C_{\tau_t}$. We assume the following.
\begin{itemize}
    \item Finite second moments.
    $\mathbb E\|W_0\|_2^2<\infty$ and $\mathbb E\|Z_\tau\|_2^2<\infty$,
    so the initialization and task exact fits are controlled in mean square.
    \item Bounded curvature.
    For some finite $L$, $\|C_\tau\|_{\mathrm{op}}\le L$ almost surely.
    Tasks may still have different ranks and curvatures.
    \item Collective coverage.
    $\overline C\succ0$, so every nonzero parameter direction affects task
    predictions with positive probability. Directions invisible to almost
    every task can be removed without changing task losses.
\end{itemize}

For a finite collection of realizable tasks, the second moment of the task
exact fits and the curvature bound are automatically finite. Collective
coverage can be obtained by restricting to the parameter subspace visible to
the tasks. Under these conditions, the recursion has a unique stationary
distribution $\nu$ with finite second moment, and the distribution of $W_t$
converges to $\nu$
\citep{brandt1986stochastic,diaconis1999iterated}.
We write $W_\infty\sim\nu$ to describe the parameter distribution after
long-term training. Proofs are given in Appendix~\ref{app:setup-results}.

\subsection{Forgetting and population loss}

In this sequential learning process, forgetting measures changes in the
model's performance on past tasks, while population loss measures its average
performance over the task distribution. Since past tasks also come from
$\Pi$, we define both quantities at the distribution level and study how they
evolve during training.

Let $W_T$ denote the parameter after the first $T$ tasks. Each past task $s<T$
had zero loss immediately after training, so its current loss
$\ell_{\tau_s}(W_T)$ is its increase in loss since that time. Averaging over
past tasks and then taking expectation over the training process gives
distribution-level forgetting
\citep{lopez2017gradient,evron2022catastrophic,xu2026order}:
\[
F_T=\mathbb E\!\left[
\frac{1}{T-1}\sum_{s=1}^{T-1}\ell_{\tau_s}(W_T)
\right],\qquad T\ge 2.
\]
For a fixed parameter $w$, population loss is defined as
\[
R(w)=\mathbb E_{\tau\sim\Pi}[\ell_\tau(w)].
\]
The expected population loss after $T$ sequential updates is therefore
\[
G_T=\mathbb E[R(W_T)],
\]
where $R(W_T)$ evaluates the current parameter on a task drawn independently
from $\Pi$. The corresponding stationary population loss is
\[
G_\infty:=\mathbb E[R(W_\infty)].
\]

Both $F_T$ and $G_T$ use the same task loss and evaluate tasks from $\Pi$.
This common basis allows us to relate forgetting to overall performance.
Section~\ref{sec:forgetting-loss} proves that, under our assumptions, their
gap vanishes during training and they converge to the same long-term limit.

\subsection{Jester as an instance of the task model}
\label{sec:experimental-setup}

Jester records partial ratings from $73{,}421$ users on $100$ jokes
\citep{goldberg2001eigentaste}. Users have different preferences, so we
partition the data into tasks by user. Each task contains the jokes rated by
one user and their ratings. We retain the original observation pattern and
rescale ratings to $[-1,1]$.

To predict a rating, we need to encode which joke is being evaluated.
One-hot encoding is a straightforward way to encode this discrete identity.
We use the joke's one-hot encoding as input to a shared linear model that
predicts its rating, with mean-squared error over observed ratings as the
task loss. With this encoding, each joke corresponds to an independent input
direction, so the model can fit any user's observed ratings exactly. This
satisfies the theoretical condition that each task is individually
realizable. At the same time, different users can give different ratings to
the same joke, naturally creating conflicts between tasks.

At each arrival, we sample a user independently and uniformly, then run
gradient descent on their data from the preceding parameter until an exact
fit is reached. This matches the model and training procedure in
Section~\ref{sec:projection-dynamics}, so the projection dynamics apply directly.
Partitioning tasks by user, encoding jokes with discrete inputs, and learning
ratings through squared loss together connect the real data to the
theoretical model. This allows us to calculate the long-term loss, the
sequential price, and the effects of EWC directly from the data.

The sets and numbers of jokes rated vary across users, producing differences
in task curvature. Section~\ref{sec:sequential-price} examines how these
differences affect the exact twofold-loss result, and
Section~\ref{sec:ewc} analyzes how EWC changes the sequential price.
Implementation details are in Appendix~\ref{app:jester-details}.

\section{The Common Limit of Forgetting and Population Loss}
\label{sec:forgetting-loss}

Forgetting $F_T$ and population loss $G_T$, defined in
Section~\ref{sec:setup}, both evaluate the model's performance on tasks from
the same distribution. As tasks continue to arrive, the parameter
distribution approaches stationarity. We now show that both evaluation
quantities also converge to the same stationary limit $G_\infty$.

\begin{theorem}[Asymptotic Equivalence of Forgetting and Population Loss]
\label{thm:forgetting-loss}
Under the assumptions of Section~\ref{sec:projection-dynamics}, let
$W_\infty\sim\nu$ and write
$G_\infty=\mathbb E[R(W_\infty)]$ for the stationary population loss. Then
\[
F_T\longrightarrow G_\infty,
\qquad
G_T\longrightarrow G_\infty.
\]
Moreover,
\[
|F_T-G_T|=O(T^{-1}).
\]
\end{theorem}

This result connects the long-term behavior of forgetting to the model's
overall performance. When tasks conflict, persistent loss on past tasks can
be characterized through stationary loss over the full task distribution.
To understand how much of long-term forgetting comes from task
incompatibility and how much comes from sequential training, we can
therefore analyze the components of $G_\infty$. This is the subject of the
next section.

\paragraph{Proof sketch.}
First, consider two learning processes driven by the same task sequence,
one starting from $W_0$ and the other initialized from the stationary
distribution $\nu$. Under collective coverage, subsequent updates make the
difference between their parameters converge to zero in mean square.
Since the second process has distribution $\nu$ at every time, the bounded
curvature and moment conditions then imply that the first process's expected
population loss converges to $G_\infty$.

Next, we connect evaluation on past tasks to population loss. Fix a past
task $\tau_s$ and replace its occurrence in the training sequence with a
fresh task drawn independently from $\Pi$, keeping all other tasks the same.
Denote the resulting final parameter by $\widetilde W_T^{(s)}$.
Because the original and replacement tasks come from the same distribution,
$\widetilde W_T^{(s)}$ has the same distribution as $W_T$ and is independent
of the original task $\tau_s$. Evaluating the modified model on the original
task therefore gives exactly the expected population loss,
\[
\mathbb E\!\left[\ell_{\tau_s}(\widetilde W_T^{(s)})\right]=G_T.
\]

The replacement changes only one task update. Thereafter, the two processes
use identical tasks, so the parameter difference caused by the replacement
diminishes under subsequent updates. For earlier tasks, the losses of the
original and modified models are already close. More recent tasks account
for a diminishing fraction of the full history. Averaging over all past
tasks therefore makes the gap between $F_T$ and $G_T$ vanish.
The full proof is given in Appendix~\ref{app:proof-forgetting-loss}.

\begin{figure}[H]
  \centering
  \includegraphics[width=0.46\textwidth]{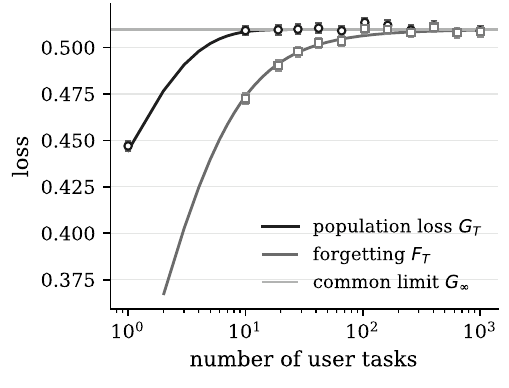}
  \caption{Forgetting and population loss approach the same limit on
    the Jester task distribution from Section~\ref{sec:experimental-setup}.
    Curves are exact finite-distribution calculations; markers average
    $2{,}000$ independent task streams. Error bars denote standard errors
    and are mostly smaller than the markers.}
  \label{fig:jester-forgetting}
\end{figure}

\section{Intrinsic Loss and the Sequential Price}
\label{sec:sequential-price}

The preceding section characterizes the long-term population loss of
sequential learning. To measure the cost of sequential training, we consider
training jointly on all tasks encountered so far and compare the long-term
performance of the two procedures.

Joint training selects parameters that minimize the average loss over all
observed tasks,
\[
\widehat W_T^{\mathrm{joint}}
\in\arg\min_w\frac1T\sum_{t=1}^T\ell_{\tau_t}(w).
\]
This is the classical empirical risk minimization problem
\citep{vapnik1991principles}. In our setting, we obtain the following
characterization.

\begin{theorem}[Consistency of Joint Training]
\label{thm:joint-consistency}
Under the assumptions of Section~\ref{sec:projection-dynamics},
\[
R(\widehat W_T^{\mathrm{joint}})
\longrightarrow\inf_w R(w)
\qquad\text{almost surely}.
\]
For a finite task distribution, this conclusion also holds for joint
empirical risk minimization in nonlinear models.
\end{theorem}

The result follows primarily from the law of large numbers and the
corresponding regularity conditions. Full conditions and proofs are given in
Appendix~\ref{app:joint-consistency}.

Thus, joint training asymptotically attains the smallest population loss
achievable by the model. We denote this optimum by $R^\star$ and call it the
\emph{intrinsic loss}. It also lower-bounds the population loss attainable
by any learning procedure, providing a baseline for sequential learning.

We can therefore decompose the stationary population loss as
\[
G_\infty=R^\star+\Delta_{\mathrm{seq}},
\qquad
\Delta_{\mathrm{seq}}:=G_\infty-R^\star\ge0.
\]
We call the loss beyond the intrinsic loss the \emph{sequential price}.
It measures the additional cost of sequential learning relative to the
asymptotically optimal performance of joint training. This definition applies
to linear and nonlinear models whenever expected population loss under
sequential learning has a long-term limit. We now return to the linear
exact-fitting process in Section~\ref{sec:projection-dynamics} to analyze
the sources and magnitude of this cost.

\subsection{Sources of the sequential price}

Let $w_\star$ be the population-loss minimizer. After sequential learning
reaches stationarity, parameters can still change as new tasks arrive.
Their deviation from the optimum can be separated into displacement of their
mean and fluctuations around that mean.

\begin{theorem}[Decomposition of the Sequential Price]
\label{thm:sequential-price}
Under the assumptions of Section~\ref{sec:projection-dynamics}, let
\[
\mu=\mathbb E[W_\infty],
\qquad
\Sigma=\operatorname{Cov}(W_\infty).
\]
Then
\[
\Delta_{\mathrm{seq}}
=
\underbrace{\|\mu-w_\star\|_{\overline C}^2}_{\text{loss from mean displacement}}
+
\underbrace{\operatorname{tr}(\overline C\Sigma)}_{\text{loss from parameter fluctuations}},
\]
where $\|v\|_{\overline C}^2=v^\top\overline Cv$.
Together with Theorem~\ref{thm:forgetting-loss}, this gives
\[
F_T\longrightarrow G_\infty=R^\star+\Delta_{\mathrm{seq}}.
\]
\end{theorem}

The stationary mean and covariance can be computed from the task
distribution. The equations and proof are given in
Appendix~\ref{app:proof-sequential-price}.

Sequential training can cause both parameter fluctuations and displacement
of the long-term mean from the optimum. The following example illustrates
the latter effect.

Consider two equally likely tasks whose target parameters are $0$ and $1$,
with the second task assigning four times as much loss to the same magnitude
of parameter error,
\[
\ell_1(w)=w^2,\qquad \ell_2(w)=4(w-1)^2.
\]
Sequential exact fitting reaches the current task's optimum at every step.
In stationarity, the parameter is therefore $0$ or $1$ with equal
probability, giving a mean of $1/2$. Joint training minimizes
\[
R(w)=\tfrac12w^2+2(w-1)^2.
\]
Because the second task is more sensitive to error, the optimal compromise
is closer to $1$, specifically $w_\star=4/5$. Thus, even removing parameter
fluctuations and predicting with the stationary mean leaves an excess loss.

This example shows that sequential learning also needs to adapt to the
relative scales of task losses. EWC penalizes parameter changes while
learning a new task, making the update magnitude depend on task curvature
and thereby allowing some adaptation to these scale differences.
Section~\ref{sec:ewc} analyzes its effect on the sequential price.

\subsection{How curvature affects the sequential price}

Task-loss scale is also reflected in curvature. For example, multiplying a
task loss by a constant multiplies its curvature by the same constant.
We therefore examine how robust the relationship between sequential price
and intrinsic loss is to curvature variation. This has a practical basis:
in Jester, analyzed below, differences in the number of ratings per user
naturally produce differences in task curvature.

\begin{corollary}[Sequential Price under Curvature Variation]
\label{cor:twofold-loss}
Under the assumptions of Section~\ref{sec:projection-dynamics}, suppose
that for some $0<m\le L<\infty$, the nonzero spectrum of each task curvature
satisfies, almost surely,
\[
\operatorname{spec}(C_\tau)\setminus\{0\}\subseteq[m,L],
\qquad \kappa:=\frac Lm.
\]
Then
\[
\frac{R^\star}{\kappa}
\le\Delta_{\mathrm{seq}}
\le\kappa R^\star,
\]
and hence
\[
\left(1+\frac1\kappa\right)R^\star
\le G_\infty
\le(1+\kappa)R^\star.
\]
In particular, if the nonzero spectra of all tasks contain only the same
constant $c>0$, then
\[
\Delta_{\mathrm{seq}}=R^\star,
\qquad G_\infty=2R^\star.
\]
\end{corollary}

The proof is given in Appendix~\ref{app:proof-twofold-loss}.

\subsection{Curvature variation in real user preferences}
\label{sec:jester-curvature}

In Jester, each task loss is the average prediction error over one user's
observed ratings. If user $i$ has rated $n_i$ jokes, the nonzero spectrum
of their task curvature is $\{1/n_i\}$.

If every user rated the same number of jokes, the curvature condition for
the exact twofold-loss result would hold, while users could still rate
different jokes and give different scores. Thus, a common nonzero curvature
can arise from natural data collection and loss definitions.

In the actual data, the number of jokes rated per user ranges from $15$ to
$100$. Applying Theorem~\ref{thm:sequential-price} gives
\[
R^\star=0.25695,\qquad
\Delta_{\mathrm{seq}}=0.25276,\qquad
G_\infty=0.50971.
\]
Despite variation in task curvature, the long-term population loss under
sequential exact fitting is approximately $1.98$ times the joint-training
optimum.

Further decomposition gives approximately $0.00027$ in loss from mean
displacement and $0.25249$ from parameter fluctuations. Thus, the excess loss
in this instance comes primarily from parameters continuing to fluctuate
with incoming user tasks. The next section studies how EWC can reduce this
excess loss.

\section{EWC and the Sequential Price}
\label{sec:ewc}

Section~\ref{sec:sequential-price} shows that sequential exact fitting incurs
an additional cost beyond the intrinsic loss. We now use EWC to study how
much of this loss a continual-learning method can remove
\citep{kirkpatrick2017overcoming,schwarz2018progress,heckel2022provable}.

Sequential exact fitting requires the model to fit the current task
completely. EWC penalizes changes from the preceding parameter while learning
a new task, allowing a compromise between fitting the current task and
retaining the previous parameter. We use the average task curvature
$\overline C$ to construct this penalty and keep its strength $\lambda>0$
fixed throughout the task sequence,
\[
W_t^{(\lambda)}
=\arg\min_w\left\{
\ell_{\tau_t}(w)
+\lambda\|w-W_{t-1}^{(\lambda)}\|_{\overline C}^2
\right\}.
\]
Here $\lambda$ controls how strongly parameter changes are constrained, and
$\overline C$ weights different directions according to their average
task curvature.

Under the assumptions of Section~\ref{sec:projection-dynamics}, this update
has a unique stationary distribution with finite second moment. Denote its
stationary population loss by $G_\lambda$. The sequential price under EWC
is then $G_\lambda-R^\star$. For a finite task distribution, this loss can
be computed directly; the equations are given in
Appendix~\ref{app:ewc-exact}.

Increasing the regularization strength reduces the change made by each task
update and can require the model to encounter more tasks before the influence
of initialization subsides. We use $T_{\mathrm{mix}}(\eta,\lambda)$ for the
number of tasks sufficient to reduce this influence, in mean square, to an
$\eta$ fraction of its initial magnitude, where $0<\eta<1$.
The precise definition is given in Appendix~\ref{app:ewc-coupling}.

\begin{theorem}[Sequential Price under EWC]
\label{thm:ewc-price-rate}
Under the assumptions of Section~\ref{sec:projection-dynamics}, for fixed
$\eta\in(0,1)$, as $\lambda\to\infty$,
\[
G_\lambda-R^\star
=\frac{a_\Pi}{\lambda}+O(\lambda^{-2}),
\]
\[
T_{\mathrm{mix}}(\eta,\lambda)
=\frac{\lambda}{2}\log\frac1\eta+O(1).
\]
The coefficient $a_\Pi$ depends on the task distribution and is positive
whenever $R^\star>0$. Its expression and the proof are given in
Appendix~\ref{app:proof-ewc-price-rate}.
\end{theorem}

Thus, stronger regularization lets EWC bring long-term population loss
closer to the intrinsic loss $R^\star$ attained by joint training.
In the strong-regularization regime, doubling the regularization strength
approximately halves the sequential price and doubles the number of tasks
sufficient to reduce the mean-square influence of initialization to the same
fraction of its initial value.

In terms of the decomposition in Section~\ref{sec:sequential-price},
strong regularization reduces both the displacement of the stationary mean
and parameter fluctuations. The loss from mean displacement is $O(\lambda^{-2})$, while the
fluctuation contribution has leading term $a_\Pi/\lambda$.
The remaining sequential price in this regime therefore comes primarily
from parameter fluctuations as new tasks continue to arrive.

\paragraph{Jester example.}
The Jester construction in Section~\ref{sec:experimental-setup} satisfies
the theoretical conditions above. We use the average curvature over all user
tasks as the EWC penalty weight. At $\lambda=16$, stationary population loss
falls from $0.5097$ under sequential exact fitting to $0.2718$, close to the
intrinsic loss $R^\star=0.2569$. The sequential price decreases from
$0.2528$ to approximately $0.0149$. Meanwhile, the number of tasks needed to
reduce the mean-square influence of initialization to $1\%$ rises from
$16$ to $53$.

This example shows that EWC can remove most of the additional loss from
sequential training when tasks conflict through real user preferences,
at the cost of a longer time for initialization dependence to decay.
Results across regularization strengths follow the two trends in the
theorem, as shown in Figure~\ref{fig:jester-ewc}(a)--(b).

\begin{figure}[H]
  \centering
  \begin{minipage}{0.40\textwidth}
    \centering
    (a) Stationary population loss\\[2pt]
    \includegraphics[width=\linewidth]{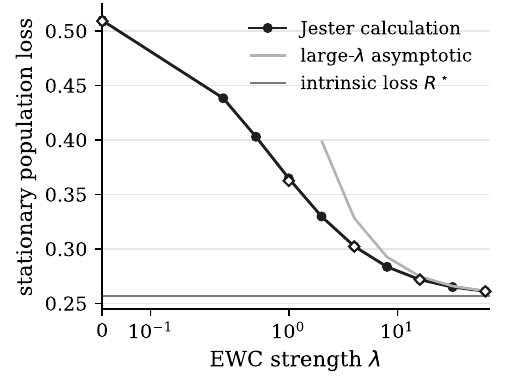}
  \end{minipage}\hspace{0.03\textwidth}%
  \begin{minipage}{0.40\textwidth}
    \centering
    (b) Initialization dependence\\[2pt]
    \includegraphics[width=\linewidth]{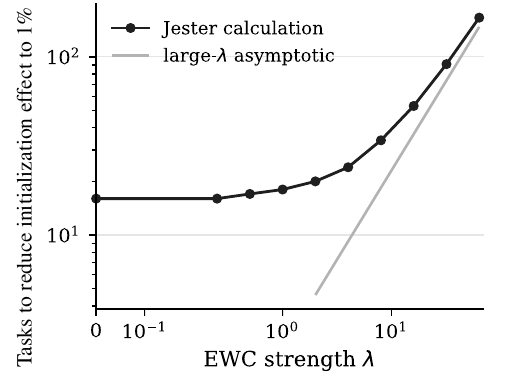}
  \end{minipage}
  \caption{EWC on Jester: (a) stationary population loss and (b) the number
  of tasks required to reduce the mean-square influence of initialization
  to $1\%$ of its initial value. Filled circles and black lines show exact
  calculations, light-gray curves show the large-$\lambda$ expressions in
  Theorem~\ref{thm:ewc-price-rate}, and the gray horizontal line in (a) marks
  $R^\star$. Open diamonds in (a) average $1{,}000$ independent task streams.
  Error bars denote standard errors and are mostly smaller than the markers.}
  \label{fig:jester-ewc}
\end{figure}

\FloatBarrier
\section{Discussion and Scope}
\label{sec:discussion}

Continual-learning tasks can arise from a relatively stable distribution,
with each task drawn independently of the preceding one. For example, online
systems may continually receive personalized tasks from a stable user
population; federated or distributed training may sample participants from a
client population in each round; and robots or agents may encounter new task
instances from a fixed environment distribution. LLM training data can likewise
be generated by independently sampling prompts, topics, or problems from a
prescribed distribution. The inputs and targets may differ substantially
across tasks even when the mechanism generating them remains relatively
stable. We model such settings with an i.i.d. task sequence to study how
sequential training favors recent tasks and how much additional loss this
incurs relative to joint training. When future tasks are selected according
to the current model, its previous performance, or a curriculum, task
generation depends on the learning history. Such self-improvement processes,
as well as settings in which the task distribution changes over time, fall
outside our analysis.

Our conclusions depend on the model assumptions to different degrees. The
proof that forgetting and population loss approach the same limit uses
i.i.d. tasks and the gradual decay of a single update's influence under
subsequent training; the explicit convergence rate follows from the
contraction of our linear recursion. Neither intrinsic loss nor the definition
of the sequential price requires a linear model: whenever sequential training
has a limiting population loss, it can be compared with the optimal population
loss. The bias--variance decomposition follows from the quadratic population loss,
while the stationary-moment calculations, curvature bounds, and EWC
asymptotics further depend on the specific update rules. The two
experiments examine these different levels. Through a natural encoding and
training procedure, the real Jester user ratings satisfy the theoretical
setting, and the calculations agree with independently sampled task streams.
The Rotated-MNIST experiment in Appendix~\ref{app:rotated-mnist} further examines
finite-step nonlinear training. We estimate historical-task loss, population
loss, and the empirical sequential price relative to a joint-training
baseline, and find that moderate functional EWC substantially reduces this
price. As regularization increases further, however, loss follows a U-shaped
curve, showing that the linear-model EWC asymptotics do not directly carry over
to this training procedure.

We first characterize the long-term limit of forgetting. Forgetting measures
the current model's average loss on past tasks, which includes both loss that
joint training cannot eliminate and additional loss caused by sequential
updates. Theorem~\ref{thm:forgetting-loss} connects forgetting to population
loss by showing that they converge to the same limit. Population loss has an
explicit benchmark, $R^\star$: the minimum population loss attainable by the
model, which joint training approaches under the conditions of
Theorem~\ref{thm:joint-consistency}. Subtracting $R^\star$ from the common limit
separates intrinsic loss from the additional cost of sequential training.
We call the latter the sequential price and study what determines it and how
it can be reduced.

We study two questions about the sequential price. First, which properties of
the task distribution determine its magnitude? We decompose it into the loss
caused by the stationary mean's displacement from the optimum and the loss
caused by persistent parameter fluctuations, then use task curvature to give
upper and lower bounds. Second, how much can a continual-learning method
reduce the sequential price, and how does this affect the decay of
initialization dependence? For EWC, we compute the long-term sequential price
and give a sufficient number of tasks for the mean-square influence of
initialization to fall to a given fraction of its initial level.

The stationary sequential price describes the additional loss after prolonged
training, while the bound on initialization decay specifies how many tasks suffice
to reduce the influence of the starting point by a given factor. Both matter when the task sequence is finite. Extending the analysis to
finite-step nonlinear training, other continual-learning methods such as
replay, and task streams that depend on the learning history or change over
time is an important direction for future work. Separating intrinsic loss
from the additional loss caused by sequential training helps clarify what
continual-learning methods can improve and what such improvements cost.

\section{Related Work}
\label{sec:related-work}

Theory of forgetting in linear or overparameterized models analyzes
worst-case, random-order, and cyclic projection dynamics for fixed task
collections \citep{evron2022catastrophic,goldfarb2023orthogonal,jeong2023cyclic},
with extensions addressing task similarity, generalization, classification,
and teacher--student dynamics
\citep{lin2023theory,ding2024understanding,asanuma2021statistical}. Most
directly, \emph{From Order to Distribution} studies a task distribution whose
tasks share a solution and characterizes the decay of distribution-level
forgetting \citep{xu2026order}. We retain this distributional viewpoint while
allowing task optima to conflict. The forgetting--loss equivalence then
identifies the generally nonzero limit, with the shared-solution result as its
zero-loss case.

Prior work improves learning by adjusting task order and optimization paths
\citep{li2025taskorder,tsipory2025greedy,hess2024complementary}, or by coordinating
gradients across tasks \citep{sener2018multitask,yu2020gradient,liu2021conflict}.
The intrinsic loss $R^\star$ provides a common benchmark for these approaches:
for a given model and task distribution, no training procedure can achieve
population loss below $R^\star$, and joint training approaches this bound under
the conditions of Theorem~\ref{thm:joint-consistency}. Thus, these approaches
can continue to improve performance while accounting for the loss that cannot
be eliminated by changing the training procedure. Evaluating continual-learning
methods by how much they reduce the sequential price helps distinguish their
improvements from the tasks' intrinsic loss, so that comparisons more directly
reflect their ability to reduce the additional cost of sequential training.

Replay methods mitigate forgetting by reusing past data, while regularization
methods constrain parameter changes during training on new tasks
\citep{lopez2017gradient,buzzega2020dark,arani2022learning}. EWC penalizes changes
according to each parameter's importance to earlier tasks, keeping important
parameters closer to their previous values
\citep{kirkpatrick2017overcoming,schwarz2018progress,benzing2022unifying}.
We analyze EWC at a fixed regularization strength to characterize how it changes
the sequential price and the decay of initialization dependence. In the
strong-regularization limit, the sequential price decreases inversely with
regularization strength, while the number of tasks sufficient to reduce the
mean-square influence of initialization by a given factor grows proportionally.

\subsection*{AI use statement}
Generative AI tools assisted with developing the theoretical formulation,
formulating and checking mathematical claims, proof development, literature
search, and manuscript drafting and editing. All AI-assisted arguments,
citations, and text were manually reviewed. The authors take responsibility
for the final content of this work, including all text and claims produced
with the aid of generative AI.

\bibliography{references}
\bibliographystyle{iclr2027_conference}

\appendix

\section{Clarifying the Scope and Contribution of the Sequential Price}
\label{app:contribution}

This work provides a concrete, computable account of the long-run cost of
continual adaptation, with distinct roles for the individual results and the
established mathematical tools. Theorem~\ref{thm:forgetting-loss} uses task
replacement and contractive coupling to show that the gap between historical
forgetting and population loss is $O(T^{-1})$, enabling direct comparison of the
long-term forgetting limit under conflicting tasks with the joint-training optimum.
Theorem~\ref{thm:sequential-price} applies the standard quadratic-loss
decomposition and moment equations for random affine recursions
\citep{wang2015exact} to separate precisely the intrinsic loss caused by task
incompatibility from the additional loss incurred by sequential exact fitting.
Building on this characterization, Corollary~\ref{cor:twofold-loss} translates
projection energy balance into a geometric relation between these two losses,
identifying conditions for the twofold-loss law and upper and lower bounds
under curvature variation; Proposition~\ref{prop:curvature-sharpness} further
shows that both bounds can be approached arbitrarily closely and provides a
task family that departs from the twofold relation. For fixed regularization
driven by population curvature, Proposition~\ref{prop:ewc-exact} and
Theorem~\ref{thm:ewc-price-rate} use stationary-moment analysis and a small-step
expansion \citep{dieuleveut2017bridging} to characterize how regularization
strength changes this additional cost and the number of tasks sufficient to
reduce the mean-square effect of initialization by a fixed factor, with explicit
coefficients determined by the task distribution. Together, these results show how established tools for random
recursions reveal concrete connections among historical forgetting, task
conflict, task geometry, and the cost of continual adaptation.

\section{Dynamics and Stationary Distribution}
\label{app:setup-results}

This section proves the statements about learning dynamics in
Section~\ref{sec:projection-dynamics}. We first remove parameter directions
that do not affect task predictions, then use collective coverage to prove
contraction and establish existence and uniqueness of the stationary law.
We write $\|v\|_A^2=v^\top A v$ and use
$\ell_\tau(w)=\|w-Z_\tau\|_{C_\tau}^2$, since every task is realizable.

\subsection{Effective parameter space}

\begin{lemma}[Invisible directions]
\label{lem:invisible-directions}
Let $C_\tau\succeq0$ and $\overline C=\mathbb{E}[C_\tau]$. Then
\[
\ker(\overline C)
=\{v\in\mathbb{R}^d:C_\tau v=0\ \text{almost surely}\}.
\]
Consequently, for every $v\in\ker(\overline C)$,
$\ell_\tau(w+v)=\ell_\tau(w)$ and $P_\tau v=v$ almost surely. Thus, if
$\overline C$ is singular, the analysis may be restricted to
$\operatorname{range}(\overline C)=\ker(\overline C)^\perp$ without changing
any task loss.
\end{lemma}

\begin{proof}
For any fixed $v$,
\[
v^\top\overline C v=\mathbb{E}[v^\top C_\tau v].
\]
The integrand is nonnegative. Hence $v^\top\overline C v=0$ if and only if
$v^\top C_\tau v=0$ almost surely, which, because $C_\tau\succeq0$, is
equivalent to $C_\tau v=0$ almost surely. The two remaining claims follow from
$\ell_\tau(w)=\|w-Z_\tau\|_{C_\tau}^2$ and from the fact that $P_\tau$ is the
orthogonal projector onto $\ker(C_\tau)$.
\end{proof}

\subsection{Contraction}

Let $\mathbb S^d$ denote the space of symmetric matrices. For a linear operator on $\mathbb S^d$, write
$\|\mathcal A\|_{\mathrm{op},F}
:=\sup_{\|M\|_F=1,\,M\in\mathbb S^d}\|\mathcal A(M)\|_F$.

\begin{lemma}[Collective coverage implies mean-square contraction]
\label{lem:coverage-contraction}
Suppose that $\overline C\succ0$. On $\mathbb{S}^d$, equipped with the
Frobenius norm, define
\[
\mathcal{S}_\Pi(A):=\mathbb{E}[P_\tau A P_\tau].
\]
Then
\[
\rho_\Pi:=\|\mathcal{S}_\Pi\|_{\mathrm{op},F}<1.
\]
In particular, $I-\mathcal{S}_\Pi$ is invertible.
\end{lemma}

\begin{proof}
For each task, the map $\mathcal{S}_\tau(A):=P_\tau A P_\tau$ is an
orthogonal projector on the matrix space $\mathbb{S}^d$ under the Frobenius
inner product. Therefore $\mathcal{S}_\Pi=\mathbb{E}[\mathcal{S}_\tau]$ is
self-adjoint, positive semidefinite, and has operator norm at most one.

Suppose for contradiction that
$\|\mathcal{S}_\Pi\|_{\mathrm{op},F}=1$. Since $\mathbb{S}^d$ is finite
dimensional, there is a symmetric matrix $A$ with $\|A\|_F=1$ such that
$\mathcal{S}_\Pi(A)=A$. Hence
\[
1
=\langle A,\mathcal{S}_\Pi(A)\rangle_F
=\mathbb{E}\|P_\tau A P_\tau\|_F^2.
\]
Because $\|P_\tau A P_\tau\|_F\leq\|A\|_F=1$, equality implies
$P_\tau A P_\tau=A$ almost surely. Choose a nonzero
$v\in\operatorname{range}(A)$. Then $P_\tau v=v$ almost surely, so
$C_\tau v=0$ almost surely. It follows that
$v^\top\overline C v=0$, contradicting $\overline C\succ0$. Thus
$\rho_\Pi<1$. The Neumann series
$(I-\mathcal{S}_\Pi)^{-1}=\sum_{k\geq0}\mathcal{S}_\Pi^k$ then converges in
operator norm.
\end{proof}

\begin{lemma}[Collective coverage implies mean contraction]
\label{lem:mean-contraction}
Suppose that $\overline C\succ0$ and let
$\overline P:=\mathbb{E}[P_\tau]$. Then
\[
\|\overline P\|_{\mathrm{op}}<1.
\]
Consequently, $I-\overline P$ is invertible.
\end{lemma}

\begin{proof}
Every $P_\tau$ is an orthogonal projector, so $\overline P$ is symmetric,
positive semidefinite, and has operator norm at most one. Suppose that its
operator norm equals one. Then there is a unit vector $v$ such that
$\overline Pv=v$, and hence
\[
1=v^\top\overline Pv
=\mathbb{E}[v^\top P_\tau v]
=\mathbb{E}\|P_\tau v\|_2^2.
\]
Since $\|P_\tau v\|_2\leq1$, equality implies
$P_\tau v=v$ almost surely. Thus $v\in\ker(C_\tau)$ almost surely and
$v^\top\overline Cv=0$, contradicting $\overline C\succ0$. Therefore
$\|\overline P\|_{\mathrm{op}}<1$, and the inverse is given by the convergent
Neumann series $(I-\overline P)^{-1}=\sum_{k\geq0}\overline P^k$.
\end{proof}

\subsection{Stationary distribution}

\begin{corollary}[Existence and uniqueness of the stationary law]
\label{cor:stationary-law}
Under the assumptions of Section~\ref{sec:setup}, the affine recursion
$W_t=P_tW_{t-1}+Z_t$ admits a unique invariant distribution within the class of
distributions with finite second moment.
\end{corollary}

\begin{proof}
Extend the i.i.d. task sequence to integer times and set
\[
W_0^{(\infty)}
:=Z_0+
\sum_{j=1}^{\infty}
P_0P_{-1}\cdots P_{-j+1}Z_{-j}.
\]
Let $M_Z:=\mathbb{E}[Z_\tau Z_\tau^\top]$ and
$\rho_\Pi=\|\mathcal{S}_\Pi\|_{\mathrm{op},F}<1$. Since $Z_{-j}$ is
independent of the later projectors in its summand,
\[
\mathbb{E}\bigl\|P_0P_{-1}\cdots P_{-j+1}Z_{-j}\bigr\|_2^2
=\operatorname{tr}\!\left(\mathcal{S}_\Pi^j(M_Z)\right)
\leq \sqrt d\,\rho_\Pi^j\|M_Z\|_F.
\]
The series therefore converges in $L^2$. Shifting all time indices by one
shows that its law is invariant under the recursion.

For uniqueness, synchronously couple two copies initialized from any two
finite-second-moment invariant laws and use the same future task sequence.
Their difference satisfies
\[
D_t=P_tP_{t-1}\cdots P_1D_0,
\]
and hence
\[
\mathbb{E}\|D_t\|_2^2
=\operatorname{tr}\!\left(
\mathcal{S}_\Pi^t\bigl(\mathbb{E}[D_0D_0^\top]\bigr)
\right)
\longrightarrow0.
\]
The two invariant laws must therefore coincide. The same coupling applies
when one process starts from any square-integrable $W_0$, independent of the
future task sequence, and the other starts from the invariant law. Their
mean-square distance tends to zero, so the distribution of $W_t$ converges
to the invariant law.
\end{proof}

\section{The Common Limit of Forgetting and Population Loss}
\label{app:proof-forgetting-loss}

This section proves Theorem~\ref{thm:forgetting-loss}. We first couple the
learning process to a stationary process to show convergence of population
loss. We then replace a single historical task to control its dependence on
the current parameter and obtain the $O(T^{-1})$ gap between forgetting and
population loss.

\begin{proof}
Write $\rho_\Pi:=\|\mathcal{S}_\Pi\|_{\mathrm{op},F}<1$. We first show that
the population loss converges to its stationary value. Let
$W_0^{(\nu)}\sim\nu$ be independent of $W_0$ and of the task stream, and run
$W_t^{(\nu)}$ using the same tasks as $W_t$. Then
$W_t^{(\nu)}\sim\nu$ for every $t$, while
\[
D_t:=W_t-W_t^{(\nu)}=P_tP_{t-1}\cdots P_1D_0.
\]
Since $D_0$ is independent of the task stream,
\[
\mathbb{E}[D_tD_t^\top]
=\mathcal{S}_\Pi^t\!\left(\mathbb{E}[D_0D_0^\top]\right),
\]
and therefore
\[
\mathbb{E}\|D_t\|_2^2
\leq
\sqrt d\,\rho_\Pi^t
\left\|\mathbb{E}[D_0D_0^\top]\right\|_F
\longrightarrow0.
\]
In particular, $\sup_t\mathbb{E}\|W_t\|_2^2<\infty$. Using the population-loss
identity in Lemma~\ref{lem:population-optimum},
\begin{align*}
|G_t-G_\infty|
&=\left|
\mathbb{E}\|W_t-w_\star\|_{\overline C}^2
-\mathbb{E}\|W_t^{(\nu)}-w_\star\|_{\overline C}^2
\right| \\
&\leq
\|\overline C\|_{\mathrm{op}}
\bigl(\mathbb{E}\|D_t\|_2^2\bigr)^{1/2}
\left(
\mathbb{E}\|W_t+W_t^{(\nu)}-2w_\star\|_2^2
\right)^{1/2},
\end{align*}
which tends to zero.

We next compare forgetting with population loss. Fix $s<T$ and draw
$\tau_s'\sim\Pi$ independently of both $W_0$ and the entire original task
stream, with associated
$(P_s',Z_s')$. Define a modified trajectory by keeping the original tasks at
all times except $s$ and using $\tau_s'$ at time $s$; denote its states by
$\widetilde W_t^{(s)}$. The final state $\widetilde W_T^{(s)}$ has the same
distribution as $W_T$ and is independent of the original task $\tau_s$.
Consequently,
\[
\mathbb{E}\!\left[\ell_{\tau_s}(\widetilde W_T^{(s)})\right]
=\mathbb{E}[R(\widetilde W_T^{(s)})]
=G_T.
\]

At the replacement time,
\[
W_s-\widetilde W_s^{(s)}
=(P_s-P_s')W_{s-1}+Z_s-Z_s'.
\]
Because orthogonal projectors have operator norm at most one,
$\sup_t\mathbb{E}\|W_t\|_2^2<\infty$, and
$\mathbb{E}\|Z_\tau\|_2^2<\infty$, there is a finite constant $K_0$ such
that
\[
\sup_{s\geq1}
\mathbb{E}\bigl\|W_s-\widetilde W_s^{(s)}\bigr\|_2^2
\leq K_0.
\]
The difference at time $s$ is independent of the future task projectors, and
for every $T>s$,
\[
W_T-\widetilde W_T^{(s)}
=P_TP_{T-1}\cdots P_{s+1}
\bigl(W_s-\widetilde W_s^{(s)}\bigr).
\]
It follows that, for another finite constant $K_1$,
\[
\mathbb{E}\bigl\|W_T-\widetilde W_T^{(s)}\bigr\|_2^2
\leq K_1\rho_\Pi^{T-s}.
\]

For any $w,v\in\mathbb{R}^d$,
\[
\ell_{\tau_s}(w)-\ell_{\tau_s}(v)
=(w-v)^\top C_s(w+v-2Z_s).
\]
Since $\|C_s\|_{\mathrm{op}}\leq L$ almost surely, Cauchy--Schwarz and the
uniform second-moment bounds give a finite constant $K_2$, independent of
$s$ and $T$, such that
\begin{align*}
\left|
\mathbb{E}[\ell_{\tau_s}(W_T)]-G_T
\right|
&\leq
L\bigl(\mathbb{E}\|W_T-\widetilde W_T^{(s)}\|_2^2\bigr)^{1/2} \\
&\quad\times
\left(
\mathbb{E}\|W_T+\widetilde W_T^{(s)}-2Z_s\|_2^2
\right)^{1/2} \\
&\leq K_2\rho_\Pi^{(T-s)/2}.
\end{align*}
Finally,
\begin{align*}
|F_T-G_T|
&\leq
\frac{1}{T-1}
\sum_{s=1}^{T-1}
\left|
\mathbb{E}[\ell_{\tau_s}(W_T)]-G_T
\right| \\
&\leq
\frac{K_2}{T-1}
\sum_{j=1}^{T-1}\rho_\Pi^{j/2}
=O(T^{-1})
\longrightarrow0.
\end{align*}
Together with $G_T\to G_\infty$, this proves the theorem.
\end{proof}

\section{Intrinsic Loss and the Sequential Price}
\label{app:intrinsic-sequential-price}

This section gives the full proofs of the results in
Section~\ref{sec:sequential-price}. We first establish intrinsic loss and the
joint-training benchmark, then calculate the additional loss from sequential
exact fitting and prove the curvature bounds, the exact twofold-loss result,
and sharpness of the bounds.

\subsection{Intrinsic loss and task compatibility}

In the linear setting, write $h_\tau=C_\tau Z_\tau$,
$\overline h=\mathbb E[h_\tau]$, and
$w_\star=\overline C^{-1}\overline h$.
Since each task is realizable,
$\ell_\tau(w)=\|w-Z_\tau\|_{C_\tau}^2$.
The next lemma shows that $R^\star=R(w_\star)$.

\begin{lemma}[Population optimum and task compatibility]
\label{lem:population-optimum}
Under the assumptions of Section~\ref{sec:setup},
\[
R(w)=R^\star+\|w-w_\star\|_{\overline C}^2
\qquad\text{for every }w\in\mathbb{R}^d.
\]
Hence $w_\star$ is the unique minimizer of $R$. If $W$ is the possibly
random output of any learning procedure and $\widetilde\tau\sim\Pi$ is an
independent evaluation task, then
\[
\mathbb{E}[\ell_{\widetilde\tau}(W)]
=\mathbb{E}[R(W)]\geq R^\star,
\]
whenever the expectation is finite, with equality if and only if
$W=w_\star$ almost surely. Moreover,
\[
R^\star=0
\quad\Longleftrightarrow\quad
\text{there exists }w^\circ\in\mathbb{R}^d\text{ such that }
X_\tau w^\circ=y_\tau\text{ for }\Pi\text{-almost every }\tau.
\]
\end{lemma}

\begin{proof}
Using $\overline h=\overline Cw_\star$ and expanding the quadratic loss gives
\begin{align*}
R(w)
&=\mathbb{E}[Z_\tau^\top C_\tau Z_\tau]
  +w^\top\overline Cw-2w^\top\overline h \\
&=R(w_\star)+(w-w_\star)^\top\overline C(w-w_\star).
\end{align*}
Because $\overline C\succ0$, the second term is nonnegative and vanishes only
at $w=w_\star$. Applying this pointwise identity to the random vector $W$ and
then taking expectations proves the lower bound and its equality condition.
Independence of $\widetilde\tau$ gives
$\mathbb{E}[\ell_{\widetilde\tau}(W)]=\mathbb{E}[R(W)]$.

If a common solution $w^\circ$ exists almost surely, then
$R(w^\circ)=0$, and therefore $R^\star=0$. Conversely, if $R^\star=0$, then
\[
0=R(w_\star)=\mathbb{E}[\ell_\tau(w_\star)].
\]
The integrand is nonnegative, so $\ell_\tau(w_\star)=0$ almost surely, which
is equivalent to $X_\tau w_\star=y_\tau$ almost surely. Thus $w_\star$ itself
is a common solution.
\end{proof}

\subsection{Consistency of joint training}
\label{app:joint-consistency}

\begin{lemma}[Consistency of joint training]
\label{lem:joint-consistency}
Under the assumptions of Section~\ref{sec:setup}, let
\[
\widehat W_T^{\mathrm{joint}}
\in\arg\min_w\frac1T\sum_{t=1}^T\ell_{\tau_t}(w).
\]
Then, almost surely, the empirical objective has the unique minimizer
\[
\widehat W_T^{\mathrm{joint}}
=\widehat C_T^{-1}\widehat h_T
\]
for all sufficiently large $T$, where
\[
\widehat C_T:=\frac1T\sum_{t=1}^T C_t,
\qquad
\widehat h_T:=\frac1T\sum_{t=1}^T h_{\tau_t}.
\]
Moreover,
\[
\widehat W_T^{\mathrm{joint}}\longrightarrow w_\star,
\qquad
R(\widehat W_T^{\mathrm{joint}})\longrightarrow R^\star
\qquad\text{almost surely}.
\]
\end{lemma}

\begin{proof}
The assumptions $\|C_\tau\|_{\mathrm{op}}\leq L$ almost surely and
$\mathbb{E}\|Z_\tau\|_2^2<\infty$ imply
$\mathbb{E}\|C_\tau Z_\tau\|_2<\infty$. The strong law of large numbers,
applied entrywise in finite dimension, therefore gives
\[
\widehat C_T\longrightarrow\overline C,
\qquad
\widehat h_T\longrightarrow\overline h
\qquad\text{almost surely}.
\]
Since $\overline C\succ0$, continuity of the smallest eigenvalue implies that
$\widehat C_T\succ0$ for all sufficiently large $T$, almost surely. Expanding
the empirical objective then shows that its unique minimizer is
$\widehat C_T^{-1}\widehat h_T$. Continuity of matrix inversion on the
positive-definite cone yields
\[
\widehat W_T^{\mathrm{joint}}
=\widehat C_T^{-1}\widehat h_T
\longrightarrow
\overline C^{-1}\overline h=w_\star
\qquad\text{almost surely}.
\]
Finally, Lemma~\ref{lem:population-optimum} gives
\[
R(\widehat W_T^{\mathrm{joint}})-R^\star
=\|\widehat W_T^{\mathrm{joint}}-w_\star\|_{\overline C}^2
\longrightarrow0.
\]
\end{proof}

\begin{lemma}[Joint-training consistency for finite task distributions]
\label{lem:finite-task-consistency}
Let $\mathcal W$ be a nonempty parameter set and let tasks take values in
$\{1,\ldots,K\}$ with probabilities $p_i>0$. For each task, let
$\ell_i:\mathcal W\to[0,\infty)$ be finite-valued; the predictor may be
nonlinear. Define
\[
R(w)=\sum_{i=1}^Kp_i\ell_i(w),\qquad
R^\star=\inf_{w\in\mathcal W}R(w).
\]
For i.i.d. task draws, put
$\widehat R_T(w)=T^{-1}\sum_{t=1}^T\ell_{\tau_t}(w)$.
If a sequence of parameters satisfies
\[
\widehat R_T(\widehat W_T)
\le\inf_{w\in\mathcal W}\widehat R_T(w)+\varepsilon_T,
\qquad
\varepsilon_T\ge0,\quad\varepsilon_T\longrightarrow0
\quad\text{almost surely},
\]
then $R(\widehat W_T)\to R^\star$ almost surely.
In particular, the conclusion holds for any sequence of global empirical
minimizers whenever these minimizers exist.
\end{lemma}

\begin{proof}
Let $\widehat p_{T,i}$ be the empirical frequency of task $i$. The strong
law of large numbers and finiteness of the task set imply
\[
\delta_T:=\max_{1\le i\le K}
\left|\frac{\widehat p_{T,i}}{p_i}-1\right|
\longrightarrow0
\qquad\text{almost surely}.
\]
Since all task losses are nonnegative, simultaneously for every
$w\in\mathcal W$,
\[
(1-\delta_T)R(w)
\le\widehat R_T(w)
\le(1+\delta_T)R(w).
\]
Work on the probability-one event on which $\delta_T\to0$ and
$\varepsilon_T\to0$. For all sufficiently large $T$, $\delta_T<1$ and
\begin{align*}
R^\star
\le R(\widehat W_T)
&\le\frac{\widehat R_T(\widehat W_T)}{1-\delta_T}\\
&\le\frac{\inf_w\widehat R_T(w)+\varepsilon_T}{1-\delta_T}\\
&\le\frac{(1+\delta_T)R^\star+\varepsilon_T}{1-\delta_T}.
\end{align*}
The infimum $R^\star$ is finite because the task set is finite and each
loss is finite at any fixed reference parameter. The upper bound converges
to $R^\star$, proving the claim. This argument does not require attainment
of the population infimum.
\end{proof}

Lemma~\ref{lem:joint-consistency} proves the linear case of
Theorem~\ref{thm:joint-consistency}; Lemma~\ref{lem:finite-task-consistency}
proves its finite-task extension. In the nonlinear case, the theorem concerns
empirical minimizers. The more general statement above also covers
approximate minimizers whose optimization error vanishes.

\subsection{Decomposition of the sequential price}
\label{app:proof-sequential-price}

The stationary moments in Theorem~\ref{thm:sequential-price} can be computed
from
\[
\mu=(I-\mathbb E[P_\tau])^{-1}\mathbb E[Z_\tau],
\qquad
\Sigma=(I-\mathcal S_\Pi)^{-1}
\mathbb E[\xi_\tau\xi_\tau^\top],
\]
where $\mathcal S_\Pi(A)=\mathbb E[P_\tau A P_\tau]$ and
$\xi_\tau=Z_\tau-(I-P_\tau)\mu$.

\begin{proof}
Let $W_\infty\sim\nu$. Lemma~\ref{lem:population-optimum} gives the pointwise
identity
\[
R(w)=R^\star+\|w-w_\star\|_{\overline C}^2.
\]
Taking expectation at $W_\infty$ and writing
$W_\infty-w_\star=(W_\infty-\mu)+(\mu-w_\star)$ yields
\begin{align*}
G_\infty-R^\star
&=\mathbb{E}\|W_\infty-w_\star\|_{\overline C}^2 \\
&=\mathbb{E}\|W_\infty-\mu\|_{\overline C}^2
  +\|\mu-w_\star\|_{\overline C}^2 \\
&=\operatorname{tr}(\overline C\Sigma)
  +\|\mu-w_\star\|_{\overline C}^2.
\end{align*}
The cross term vanishes because
$\mathbb{E}[W_\infty-\mu]=0$. This proves
$G_\infty=R^\star+\Delta_{\mathrm{seq}}$ once the stationary moments are
identified.

Draw a fresh task $\tau\sim\Pi$ independently of $W_\infty$ and define the
one-step update
\[
W_\infty':=P_\tau W_\infty+Z_\tau.
\]
Stationarity gives $W_\infty'\overset{d}{=}W_\infty$. Taking expectations,
\[
\mu=\mathbb{E}[P_\tau]\mu+\mathbb{E}[Z_\tau].
\]
Lemma~\ref{lem:mean-contraction} shows that
$I-\mathbb{E}[P_\tau]$ is invertible, and therefore
\[
\mu
=\bigl(I-\mathbb{E}[P_\tau]\bigr)^{-1}\mathbb{E}[Z_\tau].
\]

Set $U:=W_\infty-\mu$ and
\[
\xi_\tau:=Z_\tau-(I-P_\tau)\mu.
\]
The stationary mean equation implies
\[
\mathbb{E}[\xi_\tau]
=\mathbb{E}[Z_\tau]-(I-\mathbb{E}[P_\tau])\mu
=0.
\]
After centering the one-step recursion,
\[
W_\infty'-\mu=P_\tau U+\xi_\tau.
\]
The random vector $U$ is independent of the fresh task and has mean zero.
Consequently, both cross-covariance terms vanish:
\[
\mathbb{E}[P_\tau U\xi_\tau^\top]
=\mathbb{E}_\tau\!\left[
P_\tau\,\mathbb{E}[U]\,\xi_\tau^\top
\right]
=0,
\]
and likewise for its transpose. Taking covariances gives
\[
\Sigma
=\mathcal{S}_\Pi(\Sigma)
 +\mathbb{E}[\xi_\tau\xi_\tau^\top].
\]
Lemma~\ref{lem:coverage-contraction} makes
$I-\mathcal{S}_\Pi$ invertible, so
\[
\Sigma
=\bigl(I-\mathcal{S}_\Pi\bigr)^{-1}
  \mathbb{E}[\xi_\tau\xi_\tau^\top].
\]
The interpretation of $R^\star$ as the optimum attainable by a shared
parameter and as the limit of joint training follows from
Lemmas~\ref{lem:population-optimum} and~\ref{lem:joint-consistency}.
Finally, Theorem~\ref{thm:forgetting-loss} gives
$F_T\to G_\infty=R^\star+\Delta_{\mathrm{seq}}$.
\end{proof}

\subsection{Curvature bounds and the exact twofold loss}
\label{app:proof-twofold-loss}

\begin{proof}
Let $Q_\tau=I-P_\tau$ be the projector onto the row space of $X_\tau$.
The spectral assumption is equivalent to
$mQ_\tau\preceq C_\tau\preceq LQ_\tau$.
Write $\overline Q=\mathbb E[Q_\tau]$ and
\[
r_\tau:=Q_\tau(Z_\tau-w_\star).
\]
Because $P_\tau Q_\tau=0$, the centered exact-fitting update is
\[
W_t-w_\star=P_{\tau_t}(W_{t-1}-w_\star)+r_{\tau_t},
\qquad
P_{\tau_t}r_{\tau_t}=0.
\]
Let $E=W_\infty-w_\star$ and take a fresh task $\tau$ independent of $E$.
Stationarity implies that $P_\tau E+r_\tau$ has the same distribution as
$E$. Orthogonality therefore gives
\begin{align}
\mathbb E\|E\|^2
&=\mathbb E\|P_\tau E\|^2+\mathbb E\|r_\tau\|^2,\nonumber\\
\mathbb E\|E\|_{\overline Q}^2
&=\mathbb E\|r_\tau\|^2.
\label{eq:stationary-projection-balance}
\end{align}
Here the second line uses independence and
$\|E\|^2-\mathbb E_\tau\|P_\tau E\|^2
=\|E\|_{\overline Q}^2$.

Since $C_\tau=C_\tau Q_\tau=Q_\tau C_\tau$, the intrinsic loss is
\[
R^\star=\mathbb E\|r_\tau\|_{C_\tau}^2.
\]
The assumed curvature bounds hence yield
\begin{equation}
m\,\mathbb E\|r_\tau\|^2
\leq R^\star
\leq L\,\mathbb E\|r_\tau\|^2.
\label{eq:intrinsic-active-bounds}
\end{equation}
Averaging the same bounds gives
$m\overline Q\preceq\overline C\preceq L\overline Q$. Together with
Theorem~\ref{thm:sequential-price} and
Eq.~\eqref{eq:stationary-projection-balance}, this gives
\begin{equation}
m\,\mathbb E\|r_\tau\|^2
\leq\Delta_{\mathrm{seq}}
\leq L\,\mathbb E\|r_\tau\|^2.
\label{eq:sequential-active-bounds}
\end{equation}
Comparing Eqs.~\eqref{eq:intrinsic-active-bounds}
and~\eqref{eq:sequential-active-bounds} proves the sequential-price bounds.
Adding $R^\star$ gives the bounds on $G_\infty$, and
Theorem~\ref{thm:forgetting-loss} gives $F_T\to G_\infty$.

If $C_\tau=cQ_\tau$ almost surely, then both
Eqs.~\eqref{eq:intrinsic-active-bounds}
and~\eqref{eq:sequential-active-bounds} are equalities with value
$c\,\mathbb E\|r_\tau\|^2$. Hence
$\Delta_{\mathrm{seq}}=R^\star$, $G_\infty=2R^\star$, and
$F_T\to2R^\star$.
\end{proof}

\subsection{Sharpness and additional examples}

\paragraph{A non-diagonal example.}
Let $C_\tau=c u_\tau u_\tau^\top$ with $c>0$ fixed and $u_\tau$ a random unit vector with full support on the sphere. Then $Q_\tau=u_\tau u_\tau^\top$ and $\overline C\succ0$. For realizable task targets satisfying the moment assumptions, Corollary~\ref{cor:twofold-loss} gives $G_\infty=2R^\star$. Each update replaces the component along $u_\tau$ and preserves its orthogonal complement. Designs with unbounded nonzero eigenvalues fall outside the uniform spectral assumption.

\paragraph{Tightness and departures from the twofold law.}
\begin{proposition}[Sharpness of the Curvature Bounds]
\label{prop:curvature-sharpness}
For each $\kappa\geq1$, there is a family of task distributions satisfying the
assumptions of Corollary~\ref{cor:twofold-loss}, indexed by $p\in(0,1)$, for
which $R^\star>0$ and
\[
\frac{\Delta_{\mathrm{seq}}}{R^\star}
=\frac{1-p}{\kappa}+p\kappa.
\]
Consequently, the constants $1/\kappa$ and $\kappa$ in that corollary are
optimal uniformly over the stated class of task distributions.
\end{proposition}

\begin{proof}
Consider the one-dimensional tasks
\[
\ell_0(w)=w^2,\qquad \ell_1(w)=\kappa(w-1)^2,
\]
sampled with probabilities $1-p$ and $p$, respectively. Here $Q_\tau=1$,
and the curvature assumption holds with $m=1$ and $L=\kappa$. Each exact-fit
update sets the parameter to the current task optimum, so the stationary
parameter $W_\infty$ equals $0$ or $1$ with probabilities $1-p$ and $p$.
Direct minimization of $R(w)=(1-p)w^2+p\kappa(w-1)^2$ gives
\[
w_\star=\frac{\kappa p}{1-p+\kappa p},\qquad
R^\star=\frac{\kappa p(1-p)}{1-p+\kappa p}.
\]
Evaluating the stationary parameter on an independent task gives
\[
G_\infty=(1-p)\mathbb E[W_\infty^2]
+p\kappa\mathbb E[(W_\infty-1)^2]
=(1+\kappa)p(1-p).
\]
Subtracting $R^\star$ and dividing by it yields the stated ratio. For fixed
$\kappa>1$, the ratio approaches $1/\kappa$ as $p\to0$ and $\kappa$ as
$p\to1$, while $R^\star>0$ for every $p\in(0,1)$. For $\kappa=1$, the
ratio equals $1$ for every $p$.
\end{proof}

For $\kappa=4$ and $p=1/2$, this is the two-task example in
Section~\ref{sec:sequential-price}, and $G_\infty/R^\star=3.125$.
The mechanism is visible in the weights: exact fitting visits the task optima
according to their sampling probabilities, whereas the joint-training optimum
also accounts for their curvatures. Changing the frequency of the
higher-curvature task can therefore make the sequential price smaller or
larger than the intrinsic loss, even with the active subspace held fixed.

\section{The Sequential Price under EWC}
\label{app:proof-ewc-price-rate}

This section first gives exact calculations of stationary sequential price
at any fixed EWC strength, then proves the large-strength asymptotics in
Theorem~\ref{thm:ewc-price-rate} and defines and derives the decay time of
initialization dependence.

\subsection{Exact calculations at fixed regularization strength}
\label{app:ewc-exact}

For fixed $\lambda$, we suppress the iterate superscript in the calculations
below. Write $\nu_\lambda$ for the stationary distribution established in
Proposition~\ref{prop:ewc-exact}, and
$G_\lambda=\mathbb E[R(W_\infty^{(\lambda)})]$ for
$W_\infty^{(\lambda)}\sim\nu_\lambda$.
\begin{proposition}[Exact Finite-Strength Characterization]
\label{prop:ewc-exact}
Under the assumptions of Section~\ref{sec:setup}, fixed-strength EWC admits a
unique stationary distribution within the class of distributions with finite
second moment for every $\lambda>0$. When the task distribution has finite
support, its stationary population loss can be computed
exactly by solving a finite system of linear equations determined by the task
distribution. If the
task curvatures are diagonal, the means and coordinate variances can be
computed from scalar equations, and these suffice to determine population loss.
\end{proposition}

\begin{proof}[Proof of Proposition~\ref{prop:ewc-exact}]
The first-order optimality condition is
\[
C_{\tau_t}(W_t-Z_{\tau_t})
+\lambda\overline C(W_t-W_{t-1})=0.
\]
Because $C_{\tau_t}+\lambda\overline C\succ0$, this equation has the unique
solution
\[
W_t=A_{\tau_t,\lambda}W_{t-1}
    +B_{\tau_t,\lambda}Z_{\tau_t},
\]
where
\[
A_{\tau,\lambda}
:=\lambda(C_\tau+\lambda\overline C)^{-1}\overline C,
\qquad
B_{\tau,\lambda}
:=(C_\tau+\lambda\overline C)^{-1}C_\tau.
\]
Their sum is the identity.

We first establish contraction. Introduce the transformed coordinates
\[
\widehat W:=\overline C^{1/2}W,
\qquad
\widehat Z_\tau:=\overline C^{1/2}Z_\tau,
\qquad
\widehat C_\tau:=\overline C^{-1/2}C_\tau\overline C^{-1/2}.
\]
Then $\mathbb E[\widehat C_\tau]=I$, and the transformed old-state matrix is
\[
\widehat A_{\tau,\lambda}
:=\overline C^{1/2}A_{\tau,\lambda}\overline C^{-1/2}
=(I+\widehat C_\tau/\lambda)^{-1}.
\]
It is symmetric and satisfies $0\preceq\widehat A_{\tau,\lambda}\preceq I$.
Define
\[
\rho_\lambda
:=\lambda_{\max}\!\left(
\mathbb E[\widehat A_{\tau,\lambda}^2]
\right)
=\lambda_{\max}\!\left(
\mathbb E[(I+\widehat C_\tau/\lambda)^{-2}]
\right).
\]
For two chains driven by the same tasks, their transformed difference obeys
\[
\widehat D_t=\widehat A_{\tau_t,\lambda}\widehat D_{t-1}.
\]
Independence of $\tau_t$ and $\widehat D_{t-1}$ gives
\[
\mathbb E\|\widehat D_t\|_2^2
\leq\rho_\lambda\mathbb E\|\widehat D_{t-1}\|_2^2.
\]
Moreover, $\rho_\lambda<1$. Otherwise, some unit vector $v$ would satisfy
$\mathbb E\|\widehat A_{\tau,\lambda}v\|_2^2=1$. Since every realization is
nonexpansive, this would imply $\widehat C_\tau v=0$ almost surely, contradicting
$v^\top\mathbb E[\widehat C_\tau]v=1$. Iteration gives mean-square contraction at rate $\rho_\lambda$. Together with the finite second moment of the affine forcing, this
mean-square contraction gives existence and uniqueness within the class of
invariant distributions with finite second moment.

Let $W_\infty$ have this invariant distribution, and define
\[
\mu_\lambda:=\mathbb E[W_\infty],
\qquad
\Sigma_\lambda:=\operatorname{Cov}(W_\infty).
\]
Taking expectations in the stationary recursion gives
\[
\bigl(I-\mathbb E[A_{\tau,\lambda}]\bigr)\mu_\lambda
=\mathbb E[B_{\tau,\lambda}Z_\tau].
\]
The inverse exists by the preceding contraction, so
\[
\mu_\lambda
=\bigl(I-\mathbb E[A_{\tau,\lambda}]\bigr)^{-1}
\mathbb E[B_{\tau,\lambda}Z_\tau].
\]
Define
\[
\xi_{\tau,\lambda}:=B_{\tau,\lambda}(Z_\tau-\mu_\lambda),
\qquad
\mathcal S_\lambda(X)
:=\mathbb E[A_{\tau,\lambda}XA_{\tau,\lambda}^{\top}].
\]
Draw $W_\infty^{(\lambda)}\sim\nu_\lambda$ and $\tau\sim\Pi$ independently, and set
\[
W_+^{(\lambda)}
:=A_{\tau,\lambda}W_\infty^{(\lambda)}
+B_{\tau,\lambda}Z_\tau.
\]
Then $W_+^{(\lambda)}\overset{d}{=}W_\infty^{(\lambda)}$, and after centering,
\[
W_+^{(\lambda)}-\mu_\lambda
=A_{\tau,\lambda}(W_\infty^{(\lambda)}-\mu_\lambda)
+\xi_{\tau,\lambda},
\]
where $\mathbb E[\xi_{\tau,\lambda}]=0$. The centered old parameter is
independent of the fresh task and has mean zero, so both cross-covariance
terms vanish, yielding
\[
\Sigma_\lambda
=\mathcal S_\lambda(\Sigma_\lambda)
+\mathbb E[\xi_{\tau,\lambda}\xi_{\tau,\lambda}^{\top}].
\]
To justify invertibility on the covariance space, define
\[
\widehat{\mathcal S}_\lambda(X)
:=\mathbb E[\widehat A_{\tau,\lambda}X\widehat A_{\tau,\lambda}]
\]
on symmetric matrices. This operator is self-adjoint and positive semidefinite
under the Frobenius inner product. For every symmetric $X$,
\begin{align*}
\langle X,\widehat{\mathcal S}_\lambda(X)\rangle_F
&=\mathbb E\operatorname{tr}(X\widehat A_{\tau,\lambda}
X\widehat A_{\tau,\lambda})\\
&\leq\operatorname{tr}\!\left(X^2\mathbb E[
\widehat A_{\tau,\lambda}^2]\right)
\leq\rho_\lambda\|X\|_F^2.
\end{align*}
The first inequality follows by diagonalizing $\widehat A_{\tau,\lambda}$
and applying $ab\leq(a^2+b^2)/2$ to its eigenvalues.
Hence $\|\widehat{\mathcal S}_\lambda\|_{\mathrm{op},F}\leq\rho_\lambda<1$.
The change of variables $X\mapsto\overline C^{1/2}X\overline C^{1/2}$
conjugates $\mathcal S_\lambda$ to $\widehat{\mathcal S}_\lambda$, so
$I-\mathcal S_\lambda$ is invertible. This proves the second-moment formula
\[
\Sigma_\lambda
=\bigl(I-\mathcal S_\lambda\bigr)^{-1}
\mathbb E[\xi_{\tau,\lambda}\xi_{\tau,\lambda}^{\top}].
\]
Finally,
\[
\mathbb E[R(W_\infty^{(\lambda)})]
=R^\star+\mathbb E\|W_\infty^{(\lambda)}-w_\star\|_{\overline C}^2
=R^\star+\|\mu_\lambda-w_\star\|_{\overline C}^2
+\operatorname{tr}(\overline C\Sigma_\lambda),
\]
so the last two terms give the exact stationary sequential price.

For a finitely supported task distribution, every expectation above is a finite
sum. The mean equation is a finite linear system, and vectorizing the covariance
equation produces another finite linear system with a unique solution. If the
task curvatures are diagonal, then $A_{\tau,\lambda}$ and
$B_{\tau,\lambda}$ are diagonal as well. Each coordinate mean and variance
can then be computed separately, and population loss depends only on these
quantities. Off-diagonal covariances may still be nonzero. This proves the
proposition.
\end{proof}

\subsection{Sequential price under strong regularization}
\label{app:ewc-asymptotics}

The coefficient in Theorem~\ref{thm:ewc-price-rate} is
\[
a_\Pi:=\frac12\mathbb E\!\left[
\|C_\tau(Z_\tau-w_\star)\|_{\overline C^{-1}}^2
\right],
\]
where $\|v\|_{\overline C^{-1}}^2=v^\top\overline C^{-1}v$.

\begin{proof}[Sequential-price expansion in Theorem~\ref{thm:ewc-price-rate}]
Retain the notation established in the preceding proof and define
\[
\widehat B_{\tau,\lambda}
:=\overline C^{1/2}B_{\tau,\lambda}\overline C^{-1/2},
\qquad
\widehat\mu_\lambda:=\overline C^{1/2}\mu_\lambda,
\]
\[
\widehat\Sigma_\lambda
:=\overline C^{1/2}\Sigma_\lambda\overline C^{1/2},
\qquad
\widehat\xi_{\tau,\lambda}:=\overline C^{1/2}\xi_{\tau,\lambda}.
\]
Set $\varepsilon:=\lambda^{-1}$. Matrix remainders are measured in the
Frobenius norm, and vector remainders in the Euclidean norm;
$O_{L^2}(\varepsilon^2)$ denotes a random-vector remainder whose
$(\mathbb E\|\cdot\|_2^2)^{1/2}$ norm is $O(\varepsilon^2)$.
Uniformly over tasks,
\[
\widehat A_{\tau,\lambda}
=I-\varepsilon\widehat C_\tau
+\varepsilon^2\widehat C_\tau^2+O(\varepsilon^3),
\qquad
\widehat B_{\tau,\lambda}
=\varepsilon\widehat C_\tau
-\varepsilon^2\widehat C_\tau^2+O(\varepsilon^3).
\]
The remainders are uniform because $\|C_\tau\|_{\mathrm{op}}$ is bounded and
$\overline C\succ0$. Write
$\widehat w_\star=\overline C^{1/2}w_\star$ and
$\widehat d_\tau=\widehat Z_\tau-\widehat w_\star$. Population optimality gives
\[
\mathbb E[\widehat C_\tau\widehat d_\tau]=0.
\]
The stationary mean equation can thus be written as
\[
\mathbb E[\widehat B_{\tau,\lambda}]
(\widehat\mu_\lambda-\widehat w_\star)
=\mathbb E[\widehat B_{\tau,\lambda}\widehat d_\tau].
\]
The coefficient matrix of this linear equation is
$\varepsilon I+O(\varepsilon^2)$, while its right-hand vector is
$O(\varepsilon^2)$. Hence
$\widehat\mu_\lambda-\widehat w_\star=O(\varepsilon)$.

The transformed centered forcing satisfies
\[
\widehat\xi_{\tau,\lambda}
=\varepsilon\widehat C_\tau\widehat d_\tau+O_{L^2}(\varepsilon^2).
\]
Define
\[
V_{\rm noise}:=\mathbb E\!\left[
\widehat C_\tau\widehat d_\tau\widehat d_\tau^\top
\widehat C_\tau
\right].
\]
The covariance equation and the expansion of $\widehat A_{\tau,\lambda}$ give
\[
\widehat\Sigma_\lambda
=\mathbb E[\widehat A_{\tau,\lambda}
\widehat\Sigma_\lambda\widehat A_{\tau,\lambda}]
+\varepsilon^2V_{\rm noise}+O(\varepsilon^3).
\]
More explicitly, define the operator on symmetric matrices by
\[
\widehat{\mathcal S}_\lambda(X)
:=\mathbb E[\widehat A_{\tau,\lambda}X\widehat A_{\tau,\lambda}].
\]
It satisfies
\[
\mathbb E[\widehat A_{\tau,\lambda}X
\widehat A_{\tau,\lambda}]
=X-2\varepsilon X+\mathcal R_\varepsilon(X),
\qquad
\|\mathcal R_\varepsilon(X)\|_{\mathrm F}
\leq K\varepsilon^2\|X\|_{\mathrm F},
\]
because $\mathbb E[\widehat C_\tau]=I$. Hence
\[
(I-\widehat{\mathcal S}_\lambda)^{-1}
=(2\varepsilon)^{-1}I+O(1).
\]
Applying this inverse to the covariance
$\varepsilon^2V_{\rm noise}+O(\varepsilon^3)$ yields
\[
\widehat\Sigma_\lambda=\frac{\varepsilon}{2}V_{\rm noise}+O(\varepsilon^2).
\]
Therefore the squared mean bias is $O(\varepsilon^2)$ and
\begin{align*}
\operatorname{tr}(\overline C\Sigma_\lambda)
&=\operatorname{tr}(\widehat\Sigma_\lambda)\\
&=\frac{\varepsilon}{2}
\mathbb E\|\widehat C_\tau\widehat d_\tau\|_2^2
+O(\varepsilon^2)\\
&=\frac{a_\Pi}{\lambda}+O(\lambda^{-2}).
\end{align*}
This proves the expansion of $G_\lambda-R^\star$. If
$a_\Pi=0$, then $C_\tau(Z_\tau-w_\star)=0$ almost surely, so
$\ell_\tau(w_\star)=0$ almost surely and $R^\star=0$. Thus $R^\star>0$ implies
$a_\Pi>0$.

\end{proof}

\subsection{Decay of initialization dependence}
\label{app:ewc-coupling}

Using the contraction coefficient $\rho_\lambda$ from
Appendix~\ref{app:ewc-exact}, define
\[
T_{\mathrm{mix}}(\eta,\lambda)
:=\left\lceil\frac{\log\eta}{\log\rho_\lambda}\right\rceil,
\qquad 0<\eta<1.
\]
For two trajectories $W_t,V_t$ driven by the same task sequence, with jointly
square-integrable initial states independent of that sequence, this many tasks
are sufficient to ensure
\[
\mathbb E\|W_t-V_t\|_{\overline C}^2
\le\eta\,\mathbb E\|W_0-V_0\|_{\overline C}^2
\quad\text{for }t\ge T_{\mathrm{mix}}(\eta,\lambda).
\]

\begin{proof}[Initialization decay in Theorem~\ref{thm:ewc-price-rate}]
The expansion from Appendix~\ref{app:ewc-asymptotics} gives
\[
\mathbb E[\widehat A_{\tau,\lambda}^2]
=I-2\lambda^{-1}I+O(\lambda^{-2}),
\]
and hence $\rho_\lambda=1-2/\lambda+O(\lambda^{-2})$.
The contraction proved in Appendix~\ref{app:ewc-exact} gives
\[
\mathbb E\|W_t-V_t\|_{\overline C}^2
\le \rho_\lambda^t\,
\mathbb E\|W_0-V_0\|_{\overline C}^2,
\]
which proves the stated guarantee for $T_{\mathrm{mix}}$.
Expanding $\log\rho_\lambda$ and observing that the ceiling changes the
expression by at most a constant yields
\[
T_{\mathrm{mix}}(\eta,\lambda)
=\frac{\lambda}{2}\log\frac1\eta+O(1).
\]
\end{proof}

\subsection{Other fixed quadratic penalties}
For any fixed $M\succ0$, the update remains affine, with $A=(C_\tau+\lambda M)^{-1}\lambda M$ and $B=(C_\tau+\lambda M)^{-1}C_\tau$. The stationary moment equations above apply with these matrices, including for $M=I$ (use $M^{1/2}$ for the change of coordinates). Theorem~\ref{thm:ewc-price-rate} evaluates the leading coefficient and coupling horizon for $M=\overline C$. If $M$ is estimated from an independent task sample and then held fixed, these equations apply conditional on a positive-definite estimate. Population loss is still evaluated under the target task distribution, so its sensitivity depends on the accuracy of the curvature estimate.

\section{Calculations and Experimental Details for Jester}
\label{app:jester-details}

This section gives implementation details for the Jester construction in
Section~\ref{sec:experimental-setup}, explains how the full user distribution
determines the theoretical quantities in Sections~\ref{sec:sequential-price}
and~\ref{sec:ewc}, and checks the calculations against independently sampled
task streams.

\subsection{Data and task distribution}

We use Dataset~1 from the Jester release
\citep{goldberg2001eigentaste}. Its three files contain $73{,}421$ users and
$100$ joke-rating columns; the first column records the number of ratings made
by each user. The value $99$ denotes a missing rating. We remove the first
column and retain all $N=73{,}421$ users and all $4{,}136{,}360$ observed
ratings. Dividing observed scores by ten maps the rating scale from $[-10,10]$
to $[-1,1]$. A user rates between $15$ and $100$ jokes, with mean $56.34$ and
median $52$.

User $i$ and the set $\mathcal J_i$ of jokes rated by that user define one task.
The experimental task distribution is uniform over the $N$ users. Every task
arrival is an independent draw with replacement from this distribution, so a
user may appear more than once in a task stream. The observation mask and the
ratings are kept exactly as recorded in the data.

\subsection{Theoretical quantities under sequential exact fitting}

For a one-hot joke input, the single-layer model returns the corresponding
coordinate of $w\in\mathbb R^{100}$. Write
$m_{ij}=\mathbf 1\{j\in\mathcal J_i\}$, $n_i=\sum_jm_{ij}$, and set $Z_{ij}$
to the rescaled rating when $m_{ij}=1$ and to zero otherwise. Then
\[
\ell_i(w)=\sum_{j=1}^{100}c_{ij}(w_j-Z_{ij})^2,
\qquad
c_{ij}:=\frac{m_{ij}}{n_i},
\qquad
C_i=\operatorname{diag}(c_{i1},\ldots,c_{i,100}).
\]
Exact fitting overwrites the rated coordinates and retains the unrated ones,
so $P_i=I-\operatorname{diag}(m_{i1},\ldots,m_{i,100})$. The task ranks
$n_i$ vary from $15$ to $100$. Every coordinate has positive empirical
coverage, and hence $\overline C=N^{-1}\sum_iC_i\succ0$.

Define $\overline c_j=N^{-1}\sum_i c_{ij}$. The population minimizer and
intrinsic loss are
\[
w_j^\star=\frac{\sum_i c_{ij}Z_{ij}}{N\overline c_j},
\qquad
R^\star=\frac1N\sum_{i=1}^N\sum_{j=1}^{100}
c_{ij}(Z_{ij}-w_j^\star)^2.
\]
For sequential exact fitting, let
\[
p_j:=\frac1N\sum_{i=1}^N m_{ij},
\qquad
\mu_j:=\frac{\sum_i m_{ij}Z_{ij}}{Np_j},
\qquad
\sigma_j^2:=\frac{\sum_i m_{ij}Z_{ij}^2}{Np_j}-\mu_j^2.
\]
The stationary state at coordinate $j$ is the rating supplied by the most
recent task that observed that coordinate. It therefore has mean $\mu_j$ and
variance $\sigma_j^2$. Specializing Theorem~\ref{thm:sequential-price} gives
\[
G_\infty=R^\star+
\sum_{j=1}^{100}\overline c_j
\bigl((\mu_j-w_j^\star)^2+\sigma_j^2\bigr).
\]
The three terms are respectively
\[
R^\star=0.2569488,
\qquad
\sum_j\overline c_j(\mu_j-w_j^\star)^2=0.0002675,
\qquad
\sum_j\overline c_j\sigma_j^2=0.2524891,
\]
and hence the sequential price is $0.2527566$ and
$G_\infty=0.5097054$.

With $W_0=0$, coordinate $j$ has not yet been observed after $T$ tasks with
probability $(1-p_j)^T$. Consequently,
\[
\mathbb E[W_{Tj}]=(1-(1-p_j)^T)\mu_j,
\qquad
\mathbb E[W_{Tj}^2]
=(1-(1-p_j)^T)(\sigma_j^2+\mu_j^2),
\]
which gives the exact population-loss curve in Figure~\ref{fig:jester-forgetting}.
Define coordinate $j$'s contribution to $G_\infty$ by
\[
g_{\infty,j}:=\frac1N\sum_{i=1}^N c_{ij}
\bigl[\sigma_j^2+(\mu_j-Z_{ij})^2\bigr].
\]
The corresponding forgetting curve is
\[
F_T=\sum_{j=1}^{100}g_{\infty,j}
\left(1-
\frac{(1-p_j)(1-(1-p_j)^{T-1})}{(T-1)p_j}
\right),
\qquad T\geq2.
\]

\subsection{Theoretical quantities under EWC}

For Figure~\ref{fig:jester-ewc}(a)--(b), we use the objective with the fixed
population-curvature penalty
\[
\ell_i(w)+\lambda\|w-W_{t-1}\|_{\overline C}^2.
\]
This is the EWC objective in Theorem~\ref{thm:ewc-price-rate}. Its coordinate
update has the affine form
\[
W_{tj}=a_{ij}W_{t-1,j}+b_{ij},
\qquad
a_{ij}=\frac{\lambda\overline c_j}{c_{ij}+\lambda\overline c_j},
\qquad
b_{ij}=(1-a_{ij})Z_{ij},
\]
with $a_{ij}=1$ and $b_{ij}=0$ on unrated coordinates. At $\lambda=0$, this
reduces to the exact-fit projection update above.

Let $\alpha_j=\mathbb E[a_{ij}]$, $\beta_j=\mathbb E[b_{ij}]$,
$u_j=\mathbb E[a_{ij}^2]$, $v_j=\mathbb E[a_{ij}b_{ij}]$, and
$s_j=\mathbb E[b_{ij}^2]$, where expectations are uniform over users. The
stationary moments $m_{1,j}:=\mathbb E[W_{\infty,j}^{(\lambda)}]$ and $m_{2,j}:=\mathbb E[(W_{\infty,j}^{(\lambda)})^2]$ satisfy
\[
m_{1,j}=\frac{\beta_j}{1-\alpha_j},
\qquad
m_{2,j}=\frac{2v_jm_{1,j}+s_j}{1-u_j}.
\]
Write $\Delta_\lambda:=G_\lambda-R^\star$. Its exact value for the finite
user distribution is
\[
\Delta_\lambda
=\sum_{j=1}^{100}\overline c_j
\bigl(m_{2,j}-2w_j^\star m_{1,j}+(w_j^\star)^2\bigr).
\]
Figure~\ref{fig:jester-ewc}(a) plots $G_\lambda=R^\star+\Delta_\lambda$.
Panel~(b) reports the number of tasks required for the worst coordinate's
mean-square dependence on initialization to fall below $1\%$,
\[
T_{\mathrm{mix}}(0.01,\lambda)
=\left\lceil\frac{\log(0.01)}{\log\rho_\lambda}\right\rceil,
\qquad
\rho_\lambda:=\max_j u_j.
\]
For Jester, the distributional constant in Theorem~\ref{thm:ewc-price-rate} is
\[
a_\Pi
=\frac{1}{2N}\sum_{i=1}^N\sum_{j=1}^{100}
\frac{c_{ij}^2(Z_{ij}-w_j^\star)^2}{\overline c_j}
=0.2850331.
\]
The light-gray reference curves show the large-$\lambda$ expressions
$R^\star+a_\Pi/\lambda$ in panel~(a) and
$(\lambda/2)\log(100)$ in panel~(b).

\subsection{Validation with sampled task streams}

Figure~\ref{fig:jester-forgetting} averages $2{,}000$ independently sampled task streams of length
$1{,}000$. At each displayed task index, population loss is evaluated against
the complete empirical distribution, while forgetting averages the loss on all
earlier tasks in that stream. At task $1{,}000$, the sampled population loss is
$0.50899\pm0.00258$ and the sampled forgetting is
$0.50854\pm0.00258$, where the reported uncertainties are standard errors.

Appendix Figure~\ref{fig:jester-diagnostics}(a) uses $20{,}000$ independent
streams of length $80$ to estimate the stationary population loss. The
smallest joke-coverage probability is $0.2520$, making the probability that any
fixed coordinate remains unobserved after $80$ tasks less than $10^{-10}$. The
sampled stationary population loss is $0.50947\pm0.00081$. Panel~(b) uses the $100$
coordinate-wise population and stationary means computed directly from the
empirical distribution.

\begin{figure}[H]
  \centering
  \includegraphics[width=0.82\textwidth]{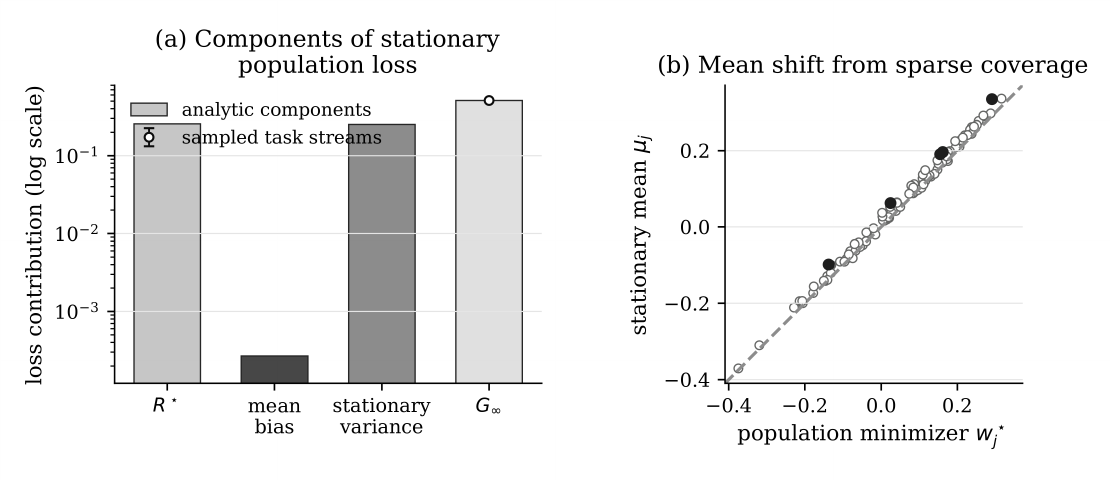}
  \caption{Additional diagnostics for sequential exact fitting on Jester.
  Panel~(a) displays the intrinsic-loss, mean-bias, and stationary-variance
  components of $G_\infty$, together with an independent task-stream estimate
  of the total. Panel~(b) compares the stationary mean $\mu_j$ with the
  population minimizer $w_j^\star$ for every joke; filled
  markers identify the five largest coordinate differences.}
  \label{fig:jester-diagnostics}
\end{figure}

We evaluate
$\lambda\in\{0,0.25,0.5,1,2,4,8,16,32,64\}$. The diamonds in Figure~\ref{fig:jester-ewc}(a)
independently average $1{,}000$ task streams of length $1{,}000$ at
$\lambda\in\{0,1,4,16,64\}$. Across these five strengths, every sampled
stationary population-loss estimate lies within $1.65$ standard errors of its exact
heterogeneous-curvature value.

\FloatBarrier

\section{Rotated MNIST under Finite-Step Nonlinear Training}
\label{app:rotated-mnist}

The Jester construction satisfies the linear exact-fitting assumptions.
Here we retain i.i.d. task sampling and use a nonlinear network trained for
finitely many steps per task to examine how the joint-training benchmark can
be used to evaluate sequential price in this setting and how functional-EWC
strength affects the additional loss.

\subsection{Tasks, model, and training configuration}

We form four task types by rotating each MNIST image. Their angles are
$0^\circ$, $45^\circ$, $90^\circ$, and $135^\circ$. At every task arrival, the
angle is sampled independently and uniformly. The task then draws a fresh
subset of $5{,}000$ examples from a pool of $50{,}000$ training images. This
construction produces an i.i.d. task stream while allowing the input
distribution, local network geometry, and learned representation to vary
across tasks.

The predictor is a multilayer perceptron with two width-$128$ ReLU hidden
layers and ten outputs. Each task uses mean-squared loss against one-hot class
targets and is trained for three epochs with AdamW, learning rate $10^{-3}$,
batch size $512$, and weight decay $0.1$. Population quantities are evaluated
on all $10{,}000$ test images under each of the four rotations. The empirical
joint-training baseline uses the same architecture and optimizer and is
trained for $200$ passes on the uniform mixture of the four task types.

\subsection{Functional EWC}

To align consolidation with the fixed population-curvature penalty in
Section~\ref{sec:ewc}, we apply the function-space objective
\[
\widehat\ell_{\tau_t}(\theta)
+\lambda\,
\mathbb E_{x\sim\mathcal D_{\rm mix}}
\bigl\|f_\theta(x)-f_{\theta_{t-1}}(x)\bigr\|_2^2,
\]
where $\mathcal D_{\rm mix}$ is the uniform mixture of the four rotated input
distributions. The expectation is estimated with fresh mixture minibatches,
and the preceding network is held fixed when evaluating its outputs. For a
linear predictor $f_W(x)=Wx$ with
$\overline C=\mathbb E_{x\sim\mathcal D_{\rm mix}}[xx^\top]$, this penalty is
exactly
\[
\lambda\operatorname{tr}
\bigl((W-W_{t-1})\overline C(W-W_{t-1})^\top\bigr),
\]
the multi-output form of the population-curvature penalty analyzed in
Section~\ref{sec:ewc}. We refer to its neural counterpart as functional EWC.
Evaluating this penalty uses unlabeled inputs from all four rotations, including
those other than the current task. This access mirrors the population curvature
used in the theory and aligns the nonlinear experiment with the analyzed
regularizer.

We evaluate
$\lambda\in\{0,0.25,0.5,1,2,4,8,16,32,64\}$ over $3{,}000$ tasks and five
random seeds. Within each seed, all strengths use the same initialization and
task-type sequence. We record every one of the first $20$ tasks and then every
$20$th task.

\subsection{Evaluation metrics and the joint-training baseline}

Let $L_a(\theta_T)$ be the test loss at rotation angle $a$ after training on task $T$, and let $\widehat p_{T-1}(a)$ be the fraction of the first $T-1$ arrivals with angle $a$. We report
\[
H_T=\sum_a\widehat p_{T-1}(a)L_a(\theta_T),\qquad
G_T=\tfrac14\sum_a L_a(\theta_T).
\]
Here $G_T$ denotes the test-set estimate for an individual run, and reported
results are aggregated across random seeds. $H_T$ evaluates the current
network on the historical mixture of rotations, while $G_T$ uses the uniform
population mixture. Finite training can leave positive acquisition loss,
which is included in $H_T$.

For each seed and $\lambda$, we estimate stationary population loss by
averaging $G_T$ from task $2{,}000$ through task $3{,}000$. The empirical sequential price is the difference from the average test loss
of the five trained joint models, which serves as the empirical comparison
baseline. Panel~(c) first smooths each trajectory over
ten recorded evaluations. Its $1\%$ band is one percent of the gap between the
task-$1$ loss and the stationary estimate. The reported relaxation time is the
first task beginning ten consecutive recorded evaluations inside this band.
All fifty trajectories satisfy this criterion within the experimental
horizon; the latest crossing occurs at task $2{,}060$.

\subsection{Results}

\begin{figure}[t]
  \centering
  \includegraphics[width=\textwidth]{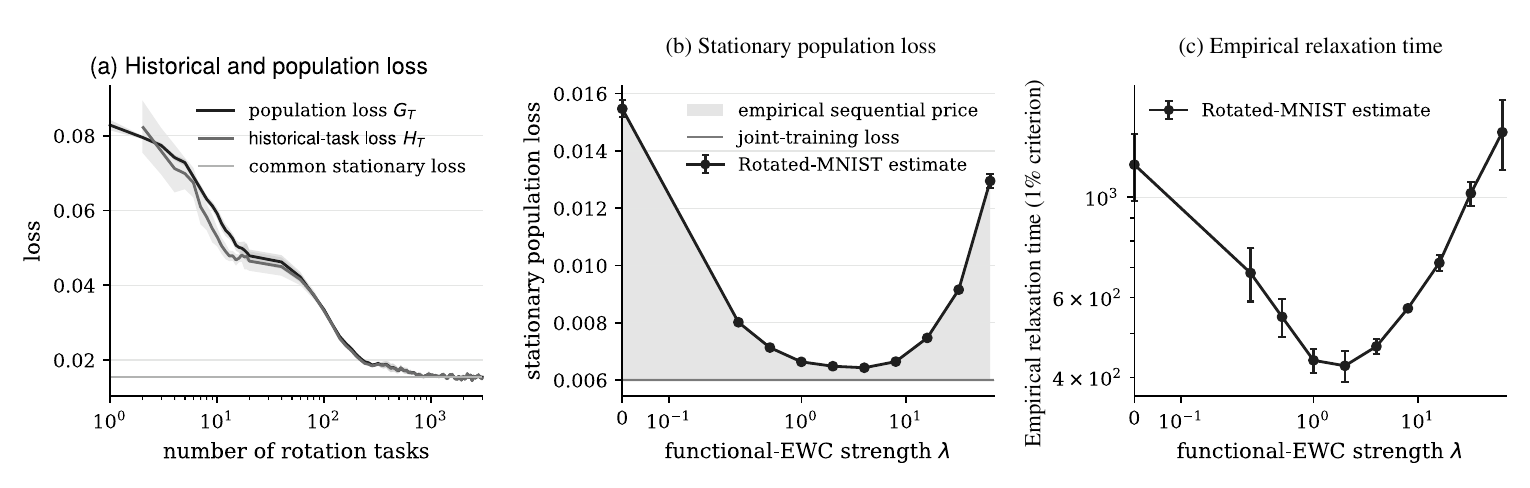}
  \caption{A nonlinear Rotated-MNIST stress test. Panel~(a) compares
  historical-task loss $H_T$ and population loss $G_T$ under sequential training ($\lambda=0$).
  Curves are five-evaluation moving averages and shaded bands denote standard
  errors across five seeds. Panel~(b) reports stationary population loss as a
  function of functional-EWC strength. The gray horizontal line is the joint-training
  loss, and the shaded gap is the empirical sequential price. Panel~(c)
  reports empirical relaxation time, the first of ten consecutive recorded evaluations for which the smoothed loss stays within the $1\%$ band defined above. Error bars in panels~(b)--(c) denote standard
  errors across seeds.}
  \label{fig:rotated-mnist}
\end{figure}

Over the final $1{,}000$ tasks, the mean absolute difference between $H_T$ and $G_T$ is $7.95\times10^{-5}$. Both quantities average the same four test losses. As the empirical rotation frequencies approach $1/4$, the two averages approach each other whenever the rotation-specific losses remain bounded.

Panels~\ref{fig:rotated-mnist}(b)--(c) reveal a U-shaped consolidation
trade-off. Without consolidation, stationary population loss is
$0.01547\pm0.00029$, compared with the joint-training loss
$0.00600\pm0.00006$. At $\lambda=4$, stationary loss falls to
$0.00643\pm0.00001$: the sequential price decreases from $0.00947$ to
$0.00043$, a $95.5\%$ reduction. Moderate consolidation also meets
the defined empirical stabilization criterion earliest; the mean relaxation times are $424\pm33$ tasks at
$\lambda=2$ and $468\pm17$ tasks at $\lambda=4$.

Stronger consolidation produces a second regime. At $\lambda=64$, stationary
population loss rises to $0.01295\pm0.00025$ and the relaxation time increases
to $1{,}392\pm245$ tasks. These results show that moderate functional EWC can substantially reduce
additional loss relative to joint training, but increasing regularization
further raises the loss. The joint-training baseline remains useful for
evaluating this change, while the strong-regularization asymptotics for the
exact dynamics in Section~\ref{sec:ewc} do not directly extend to the
finite-step nonlinear training used here.

\FloatBarrier

\section{Additional Related Work}
\label{app:related-work}

Catastrophic forgetting has been studied across task-, domain-, and
class-incremental continual learning, with broad overviews available in prior
surveys \citep{parisi2019continual,delange2021continual,wang2024survey}. The
distinctions among continual-learning scenarios and evaluation protocols are
further developed in prior work
\citep{vandeven2022three,mai2022survey,lesort2023challenging}. Representative
mitigation strategies include
replay \citep{lopez2017gradient,buzzega2020dark,arani2022learning},
importance-based regularization
\citep{kirkpatrick2017overcoming,zenke2017synaptic,aljundi2018memory}, and
recent forgetting-resistant architectures, pretrained-model adaptations, or
Bayesian updates \citep{li2023if2net,peng2025loranpac,bonnet2025bayesian}.
Empirical
analyses have also shown that the observed amount of forgetting depends on the
training regime, optimization path, model width, and evaluation assumptions
\citep{mirzadeh2020understanding,mirzadeh2021lmc,mirzadeh2022wide}. Our focus
is complementary: we quantify the additional stationary population loss induced by a
specified sequential learner in a model where that cost can be computed exactly.

The closest theoretical line studies exact-fit or overparameterized linear
continual learning through projection dynamics and task geometry.
Worst-case, cyclic, and random-order guarantees have been established for this
setting \citep{evron2022catastrophic}, followed by analyses of random orthogonal
tasks, sharper cyclic bounds, and separable classification
\citep{goldfarb2023orthogonal,jeong2023cyclic,evron2023classification}.
Subsequent work treats generalization, task similarity, and linear-regression
geometry \citep{lin2023theory,ding2024understanding,goldfarb2024joint}, as well
as the interaction between similarity and overparameterization
\citep{hiratani2024tasksimilarity,goldfarb2025overparameterization,banayeeanzade2024theoretical}.
Recent studies further examine task order, replay, and regularization schedules
\citep{li2025taskorder,tsipory2025greedy,mahdaviyeh2025replay}, together with
last-iterate convergence and optimal linear regularization
\citep{cai2024lastiterate,levinstein2025optimal,karpel2026l2}. Complementary
theories use neural-tangent overlap or teacher--student models
\citep{doan2021theoretical,lee2021continual,asanuma2021statistical}, while
other analyses study sketched or compact memory and statistical regularization
\citep{heckel2022provable,jung2025classification,jung2025memory}. Broader
perspectives use statistical physics, parabolic equations, and subspace
geometry \citep{mori2025optimal,yang2025parabolic,steele2026subspace}, while
information-theoretic approaches provide another recent viewpoint
\citep{cheng2026context}. Most directly,
\emph{From Order to Distribution} moves from a fixed task list to i.i.d. draws
from a task distribution and characterizes the decay of distribution-level
forgetting under a shared solution \citep{xu2026order}. We retain that
distributional viewpoint while allowing task optima to conflict, which changes
the state dynamics from a homogeneous projection recursion to a driven affine
recursion and makes the long-term limit the central object.

Our comparison between sequential and joint training is also related to work
on optimization objectives and access patterns. Analyses of evolving tasks,
optimization trajectories, and controlled continual-learning protocols show
that both the objective and the path used to optimize it can affect retained
performance \citep{alvarez2025evolving,hess2024complementary,mori2025optimal}.
Task ordering provides another mechanism through which sequential access
changes learning outcomes
\citep{li2025taskorder,tsipory2025greedy,goldfarb2023orthogonal}. In multi-task
learning, conflicting task objectives have been formalized through
multi-objective optimization and gradient interference
\citep{sener2018multitask,yu2020gradient,liu2021conflict}. Our comparison uses the
stationary population loss of sequential exact fitting. We use $R^\star$ for
the intrinsic loss, which joint empirical risk minimization asymptotically
attains, and define the sequential price as the additional loss of the
specified sequential learner. This separates conflict
intrinsic to the shared model from the additional cost created by the update
rule and access pattern.

Regularization-based continual learning limits movement away from parameters
that encode previous tasks. EWC weights this movement by an empirical Fisher
matrix, and online EWC consolidates the accumulated information into a running
reference state
\citep{kirkpatrick2017overcoming,schwarz2018progress,benzing2022unifying}.
Related structural regularizers refine or compress the importance geometry
\citep{zenke2017synaptic,aljundi2018memory,li2021sketched}, and theoretical
work studies when quadratic or Jacobian-based regularization can prevent
forgetting \citep{heckel2022provable,zhao2024statistical,lin2023theory}.
Levinstein et al. study vanishing forgetting for jointly realizable tasks under
fixed or increasing isotropic regularization \citep{levinstein2025optimal},
while Karpel et al. study high-dimensional generalization and horizon-dependent
isotropic $L_2$ regularization \citep{karpel2026l2}. We instead allow
conflicting tasks and characterize stationary sequential price and convergence
to stationarity under a fixed population-curvature penalty.

Technically, our recursion is connected to random affine systems and stochastic
projection methods. Contractive iterated random functions and affine
stochastic equations provide general tools for invariant laws and stationary
moments \citep{brandt1986stochastic,diaconis1999iterated}.
Randomized Kaczmarz and related projection methods study convergence under
successive compatible constraints
\citep{strohmer2009randomized,needell2010randomized,needell2014paved}, with
additional connections to alternating projections and normalized adaptive
filtering \citep{gina2018method,oswald2015convergence,sankaran2000convergence}.
We use these tools to quantify how sequential exact fitting changes the loss
induced by task conflict. The resulting analysis connects historical-task loss
to the joint-training optimum through stationary population loss, yields
curvature-dependent bounds and an exact twofold-loss law in homogeneous
settings, and quantifies the EWC price--rate trade-off.

\end{document}

%% file: math_commands.tex
\usepackage{amsmath,amsfonts,bm}

\def\eqref#1{equation~\ref{#1}}

\def\1{\bm{1}}

\DeclareMathAlphabet{\mathsfit}{\encodingdefault}{\sfdefault}{m}{sl}
\SetMathAlphabet{\mathsfit}{bold}{\encodingdefault}{\sfdefault}{bx}{n}

